%% file: main.tex
\def\isarxiv{1} %%%ICML version, we comment this line

\ifdefined\isarxiv
\documentclass[11pt]{article}
\else
\documentclass[conference]{IEEEtran}
\IEEEoverridecommandlockouts

\def\BibTeX{{\rm B\kern-.05em{\sc i\kern-.025em b}\kern-.08em
    T\kern-.1667em\lower.7ex\hbox{E}\kern-.125emX}}

\fi

\usepackage{amssymb}
\usepackage{algorithm}
\ifdefined\isarxiv
\usepackage{amsmath}
\usepackage{amsthm}
\usepackage{subfig}
\usepackage{wrapfig,epsfig}
\else 
\usepackage{amsthm}
\usepackage{times}
\fi
\usepackage{epstopdf}
\usepackage{url}
\usepackage{graphicx}
\usepackage{color}
\usepackage{epstopdf}
\usepackage{algpseudocode}
\usepackage{scrextend}
\usepackage[T1]{fontenc}
\usepackage{bbm}
\usepackage{comment}
\usepackage{multicol}
\usepackage{dsfont}
\usepackage{mathtools}
\usepackage{enumitem}
\usepackage{thmtools,thm-restate}
\usepackage{subfig}

\usepackage{tikz}
\usepackage{hyperref}
\definecolor{mydarkblue}{rgb}{0,0.08,0.45}
\hypersetup{colorlinks=true, citecolor=mydarkblue,linkcolor=mydarkblue}

\usetikzlibrary{arrows}
\ifdefined\isarxiv
\usepackage[margin=1in]{geometry}
\else 

\fi

\graphicspath{{./figs/}}

\definecolor{b2}{RGB}{51,153,255}
\definecolor{mygreen}{RGB}{80,180,0}

\newcommand{\Zhao}[1]{\textcolor{b2}{[Zhao: #1]}}
\newcommand{\Aravind}[1]{\textcolor{orange}{[Aravind: #1]}}
\newcommand{\lianke}[1]{\textcolor{red}{[Lianke: #1]}}

\newtheorem{theorem}{Theorem} %[section]
\newtheorem{lemma}[theorem]{Lemma}

\newtheorem{definition}[theorem]{Definition}

\newcommand{\wt}{\widetilde}

\renewcommand{\epsilon}{\varepsilon}

\newcommand{\N}{\mathcal{N}}
\newcommand{\R}{\mathbb{R}}

\renewcommand{\tilde}{\wt}

\newcommand{\poly}{\mathrm{poly}}

\newcommand{\E}{\mathbb{E}}
\newcommand{\Var}{\mathbf{Var}}

\newcommand {\mat}      [1] {\left[\begin{array}{#1}}
\newcommand {\rix}          {\end{array}\right]}

\newcommand{\X}{\mathcal{X}}

\DeclareMathOperator*{\Median}{median}

\makeatletter %%%%%%%%%Zhaozhuo: I add this to align the authors
\newcommand{\linebreakand}{%
  \end{@IEEEauthorhalign}
  \hfill\mbox{}\par
  \mbox{}\hfill\begin{@IEEEauthorhalign}
}
\makeatother

\begin{document}

%\maketitle
\title{Online Adaptive Mahalanobis Distance Estimation \thanks{Full version of this paper is available at~\cite{full} due to space limit.}}
\ifdefined\isarxiv

\author{
Lianke Qin\thanks{\texttt{lianke@ucsb.edu}. UC Santa Barbara.}
\and 
Aravind Reddy\thanks{\texttt{aravind.reddy@cs.northwestern.edu}. Northwestern University.}
\and 
Zhao Song\thanks{\texttt{magic.linuxkde@gmail.com}. Adobe Research.} 
}
\date{}%{\today}

\else 

% \title{Online Adaptive Mahalanobis Distance Estimation} 

\author{\IEEEauthorblockN{Lianke Qin}
\IEEEauthorblockA{\textit{Department of Computer Science} \\
\textit{University of California, Santa Barbara}\\
Santa Barbara, CA \\
lianke@ucsb.edu}
\and
\IEEEauthorblockN{Aravind Reddy}
\IEEEauthorblockA{
\textit{Department of Computer Science}\\
\textit{Northwestern University}\\
Evanston, IL \\
aravind.reddy@cs.northwestern.edu}
\and
\IEEEauthorblockN{Zhao Song}
\IEEEauthorblockA{
\textit{Adobe Research}\\
\textit{Adobe}\\
San Jose, CA \\
zsong@adobe.com}
}

\fi

\ifdefined\isarxiv
\begin{titlepage}
  \maketitle
  \begin{abstract}
\input{abstract}

  \end{abstract}
  \thispagestyle{empty}
\end{titlepage}

{\hypersetup{linkcolor=black}
\tableofcontents
}
\newpage

\else
\maketitle
\begin{abstract}
\input{abstract}
\end{abstract}

\fi

\input{intro}

\input{relat}

\input{prelim}

\input{tech}

\input{sketching}

\input{alg}

% \input{exp}
\input{experiment}
\input{concl}

%\clearpage
\ifdefined\isarxiv
 \addcontentsline{toc}{section}{References}
 \bibliographystyle{alpha}
 \bibliography{ref}
\else
 \bibliography{ref}
 %%%Zhao: This is ML First author last name et al style,
% \bibliographystyle{plainnat}
%%%Zhao: This is number style
%\bibliographystyle{plain} %%% Zhao: For this paper, let's use this one
%%%Zhao: This is TCS ABC+12 style
% \bibliographystyle{alpha}
\bibliographystyle{IEEEtran} %%% Zhao: For this paper, let's use this one

\fi

% \input{checklist}

% \newpage
% \onecolumn
% \appendix
% \section*{Appendix}
\input{app_tools}

\input{jl_sketch_app}

\input{time_app}

\input{sample_app}
\input{exp_app} 

\end{document}

%% file: abstract.tex
Mahalanobis metrics are widely used in machine learning in conjunction with methods like $k$-nearest neighbors, $k$-means clustering, and $k$-medians clustering. Despite their importance, there has not been any prior work on applying sketching techniques to speed up algorithms for Mahalanobis metrics. In this paper, we initiate the study of dimension reduction for Mahalanobis metrics. In particular, we provide efficient data structures for solving the Approximate Distance Estimation (ADE) problem for Mahalanobis distances. We first provide a randomized Monte Carlo data structure. Then, we show how we can adapt it to provide our main data structure which can handle sequences of \textit{adaptive} queries and also online updates to both the Mahalanobis metric matrix and the data points, making it amenable to be used in conjunction with prior algorithms for online learning of Mahalanobis metrics.

%% file: intro.tex
\section{Introduction}

% https://arxiv.org/pdf/2010.11252.pdf
% ITML paper \cite{dkjsd07}.

The choice of metric is critical for the success of a wide range of machine learning tasks, such as $k$-nearest neighbors, $k$-means clustering, and $k$-medians clustering. In many applications, the metric is often not provided explicitly and needs to be learned from data \cite{ssn04, dkjsd07, jkdg08}. The most common class of such learned metrics is the set of Mahalanobis metrics, which have been shown to have good generalization performance. For a set of points $\mathcal{X} = \{x_1, x_2, \dots, x_n\} \in \R^d$, a Mahalanobis metric $d$ on $\X$ is characterized by a positive semidefinite matrix $A \in \R^{d \times d}$ such that $d(x_i, x_j) = \sqrt{(x_i - x_j)^\top A (x_i - x_j)}$. Mahalanobis distances are generalizations of traditional Euclidean distances, and they allow for arbitrary linear scaling and rotations of the feature space.

%\Danyang{cite a list of papers about the applications of Mahalanobis distance.}

In parallel to the large body of work on learning Mahalanobis metrics, the use of sketching techniques has also become increasingly popular in dealing with real-world data that is often high-dimensional data and also extremely large in terms of the number of data points. In particular, a large body of influential work has focused on sketching for Euclidean distances \cite{jl84,ams99,ccf02}. Despite the importance of Mahalanobis metrics in practical machine-learning applications, there has not been any prior work focusing on dimension reduction for Mahalanobis distances.
%Sketching is known to provide speedup via dimension reduction for many applications, such as \Danyang{cite many things here}. It is thus worthwhile to consider whether we can use sketching to also improve algorithms based on Mahalanobis metrics. \Danyang{need a fancy sentence to talk about why applying sketching to algorithms for Mahalanobis distance is fundamentally different from existing approaches.}

In this paper, we initiate the study of dimension reduction for Mahalanobis distances. In particular, we focus on the Approximate Distance Estimation (ADE) problem \cite{cn20} in which the goal is to build an efficient data structure that can estimate the distance from a given query point to a set of private data points, even when the queries are provided adaptively by an adversary. ADE has many machine learning applications where the input data can be easily manipulated by users and the model accuracy is critical, such as network intrusion detection~\cite{bcm+17, cbk09}, strategic classification~\cite{hmpw16}, and autonomous driving~\cite{pmg16, lcls17, pmg+17}. We formulate the ADE problem with a Mahalanobis distance as:
\begin{definition}[Approximate Mahalanobis Distance Estimation]
For a Mahalanobis distance characterized by a positive semi-definite matrix $A\in \R^{d \times d}$, given a set of data points $\X=\{x_1,x_2,\dots,x_n\} \subset \mathbb{R}^{d}$ and an accuracy parameter $\epsilon \in(0,1)$, we need to output a data structure $\mathcal{D}$ such that: Given only $\mathcal{D}$ stored in memory with no direct access to $\X$, our query algorithm must respond to queries of the form $q \in \mathbb{R}^{d}$ by reporting distance estimates $\tilde{d}_{1}, \ldots, \tilde{d}_{n}$ satisfying
\begin{align*}
    \forall i \in[n],(1-\epsilon)\|q-x_{i}\|_{A} \leq \tilde{d}_{i} \leq(1+\epsilon)\|q-x_{i}\|_{A} .
\end{align*}
\end{definition}

%We provide the first data structure design and algorithms that satisfy this definition. \Danyang{some fancy words about how to use sketching.}

We also consider an online version of the problem, where the underlying Mahalanobis distance changes over iterations. This is an important problem, because the distance metric $A$ is learned from data and can change over time in many applications of Mahalanobis distances. For example, for an Internet image search application, the continuous collection of input images changes the distance metric $A$. We formulate our online version of the ADE problem with Mahalanobis distances as follows:

\begin{definition}[Online Adaptive Mahalanobis Distance Estimation]\label{def:main_problem}
We need to design a data structure that efficiently supports any sequence of the following operations:
\begin{itemize}
    \item \textsc{Initialize}$(U \in \R^{k \times d}, \{ x_1, x_2, \cdots, x_n \} \subset \R^d, \epsilon \in (0,1), \delta \in (0,1))$. The data structure takes $n$ data points $\{x_1, x_2, \dots, x_n\}\subset \R^d$, a $k \times d$ matrix $U$ (which defines the metric with $A = U^\top U$), an accuracy parameter $\epsilon$ and a failure probability $\delta$ as input.
    \item \textsc{UpdateU}$(u \in \R^d, a \in [k])$. Updates the $a$-th row of $k \times d$ matrix $U$.
    \item \textsc{UpdateX}$(z \in \R^d, i \in [n])$. Update the data structure with the $i$-th new data point $z$.
    % \item \textsc{Query}$(q \in \R^d, i \in [n])$. Outputs the estimated distance $\tilde{d}_i \in \R$ such that $(1-\epsilon)\|q-x_{i}\|_{A} \leq \tilde{d}_{i} \leq(1+\epsilon)\|q-x_{i}\|_{A}$ with probability at least $1 -\delta$. Note that $A = U^\top U$.
    \item \textsc{QueryPair}$(i,j  \in [n])$ Outputs a number $p$ such that $(1-\epsilon) \| x_i - x_j \|_A \leq p \leq (1+\epsilon) \cdot \| x_i - x_j \|_A $ with probability at least $1 -\delta$.
    \item \textsc{QueryAll}$(q \in \R^d)$. Outputs a set of distance estimates $\{\tilde{d}_1, \cdots, \tilde{d}_n\} \subset \R$ such that $\forall i \in[n], (1-\epsilon)\|q-x_{i}\|_{A} \leq \tilde{d}_{i} \leq(1+\epsilon)\|q-x_{i}\|_{A}$. with probability at least $1 -\delta$.
    % \lianke{below several functionalities are new}
    % \item \textsc{SubSum}$(q \in \R^d, i \in [n], j \in [n] )$. We promise that $i < j$. We want to output $\sum_{l=i}^{j} \| q - x_l \|_A$
    % \item \textsc{SubSumApprox}$(q \in \R^d, i \in [n], j \in [n] )$. We promise that $i < j$. We want to output an approximate result $\tilde{d}$ such that $(1-\epsilon)\sum_{l=i}^{j} \| q - x_l \|_A \leq \tilde{d} \leq (1+\epsilon)\sum_{l=i}^{j} \| q - x_l \|_A \leq$
    % \item \textsc{SampleApprox}$(q \in \R^d)$. Sample an index $i \in [n]$ with probability $\wt{d}_i / \sum_{j=1}^n \wt{d}_j$
    \item \textsc{SampleExact}$(q \in \R^d )$. Sample an index $i \in [n]$ with probability $d_i / \sum_{j=1}^n d_j$
\end{itemize}
\end{definition}

Our randomized data structure(Algorithm~\ref{alg:mahalanobis_maintenance},~\ref{alg:query}) supports approximate Mahalanobis distance estimation with online update to the distance metric matrix and data points even for a sequence of adaptively chosen queries. It also supports samples an index $i \in [n]$ with probability $d_i / \sum_{j=1}^n d_j$.

\paragraph{Roadmap.}
In Section~\ref{sec:related_work}, we discuss related works in online metric learning, sketching, and adaptive distance estimation.
% In Section~\ref{sec:application}, we show how our data structure can be used in some applications.
In Section~\ref{sec:prelim}, we provide some notation used throughout the paper and also provide some background. 
%and present a brief technical overview.
In Section~\ref{sec:sketch}, we demonstrate how we can use the well-known JL sketch~\cite{jl84} to design a data-structure which can support Mahalanobis distance estimation.
In Section~\ref{sec:adaptive_maintaince}, we present our adaptive Mahalanobis distance maintenance data structure and prove its time and correctness guarantees. 
We provide an experimental evaluation of our data structure in Section~\ref{sec:eval}.
We end with our conclusion in Section~\ref{sec:conclusion}.

%% file: relat.tex
\section{Related Work}\label{sec:related_work}

\paragraph{Sketching}
Sketching is a well-known technique to improve performance or memory complexity~\cite{cw13}. It has wide applications in linear algebra, such as linear regression and low-rank approximation\cite{cw13,nn13,mm13,rsw16,swz17,alszz18, qgt+19, swz19_neurips2,djssw19,  hdw20, fqz+21,zyd+22, talla2023map, rsz22, rrs+22, gsz23, gsyz23}, training over-parameterized neural network~\cite{syz21, szz21, zhasks21, qsyz23}, empirical risk minimization~\cite{lsz19, qszz23}, linear programming \cite{lsz19, lrs+21,sy21,lszzz23,gs22,bs23}, distributed problems \cite{wz16,bwz16, jll+21}, clustering~\cite{emz21}, generative adversarial networks~\cite{xzz18}, kernel
density estimation~\cite{qrs+22},
tensor decomposition \cite{swz19_soda}, trace estimation~\cite{jpwz21}, projected gradient descent~\cite{hmr18, xss21}, matrix sensing~\cite{ydh+21, dls23_sensing, qsz23}, softmax regression \cite{mrs20, dhy+22, as23,lsz23,dls23,gsy23,ssz23}, John Ellipsoid computation \cite{ccly19,syyz22}, semi-definite programming \cite{gs22}, kernel methods~\cite{acw17, akkpvwz20, cy21, swyz21}, adversarial training~\cite{gqsw22}, cutting plane method \cite{jlsw20}, discrepancy \cite{z22}, federated learning~\cite{rpuisbga20, swyz23}, kronecker protection maintenance~\cite{syyz23_dp},  reinforcement learning~\cite{ wzd+20,ssx21},  relational database \cite{qjs+22} and attention computation~\cite{qsz23b}.

\subsection{Approximate Adaptive Distance Estimation}
There has been increasing interest in understanding threats associated with the deployment of
algorithms in potentially adversarial settings~\cite{ gss15,hmpw16, lcls17}.
 Additionally, the problem of preserving statistical validity when conducting exploratory data analysis has been extensively studied~\cite{dfh+15a, dfh+15b, dfh+15c, dssu17} where the goal is to maintain coherence with an unknown distribution from which data samples are acquired. In the application of approximate nearest neighbor, the works of~\cite{kle97, kor00, cn20} present non-trivial data structures for distance estimation which supports adaptive queries. In the context of streaming algorithms, the work of~\cite{bejwy20} designs streaming algorithms that support both adaptive queries and updates.

\subsection{Online Metric Learning}
A number of recent techniques consider the metric learning problem~\cite{sj03, gr05,fsm06}. Most works handle offline learning Mahalanobis distances, which often results in expensive optimization algorithms. POLA~\cite{ssn04} is an algorithm for online  Mahalanobis metric learning that optimizes a large-margin objective and establishes provable regret bounds. However, it requires eigenvector computation at each iteration to ensure positive definiteness, which can be time-consuming in practice. The information-theoretic metric learning approach of~\cite{dkjsd07} presents an online algorithm that eliminates eigenvector decomposition operations. LEGO~\cite{jkdg08} requires no additional work for enforcing positive definiteness and can be implemented efficiently in practice. \cite{qjzl15} leverages random projection in distance metric learning for high-dimensional data.

%% file: prelim.tex
\section{Preliminaries}\label{sec:prelim}
\subsection{Notation}
For any natural number $n$, we use $[n]$ to denote the set $\{1,2,\dots,n\}$.
We use $A^\top$ to denote the transpose of matrix $A$. For any vector $v \in \R^d$ and positive semi-definite matrix $A \in \R^{d\times d}$, we use $\|v\|_{A}$ to denote $\sqrt{v^\top A v}$. We use ${\N}(0,I)$ to denote Gaussian distribution. For a probabilistic event $f(x)$, we define ${\bf 1}\{f(x)\}$ such that ${\bf 1}\{f(x)\} = 1$ if $f(x)$ holds and ${\bf 1}\{f(x)\} = 0$ otherwise. For a vector $v \in \R^d$ and real valued random variable $Y$, we will abuse
notation and use $\Median(v)$ and $\Median(Y)$ to denote the median of the entries of $v$ and the distribution
of $Y$ respectively. We use $\Pr[]$ to denote the probability, and use $\E[]$ to denote the expectation if it exists. We use $\Var[]$ to denote the variance of a random variable. We use $\wt{O}(f)$ to denote $O(f \poly(\log f))$.

\subsection{Background}
In this section, we provide several  definitions and probability results. 

\begin{definition}[Low-Rank Mahalanobis Pseudo-Metric]
For any set of points $\X \subset \R^d$, a pseudo-metric is a function $d:\X \times \X \mapsto \R_{\geq 0}$ such that $\forall x,y,z \in \X$, it satisfies: $d(x,x) = 0, d(x,y) = d(y,x), d(x,y) \leq d(x,z) + d(z,y)$.

As in some prior work on pseudo-metric learning \cite{ssn04}, we only consider pseudo-metrics which can be written in the form $d_A(x,y) = (x-y)^\top A (x-y)$ where $A$ is a positive semi-definite matrix. We define the rank of a pseudo-metric $d_A$ to be the rank of $A$.
\end{definition}

Note that the distinction between pseudo-metrics and metrics is not very crucial for our results. In addition to the above properties, a metric satisfies the property that for any $x,y \in \X, d(x,y) = 0$ if and only if $x = y$. 

We will use the following property from \cite{cn20} for our sketching matrices $\{\Pi_j \in \R^{m \times k}\}_{j=1}^L$ where $L$ denotes the number of sketching matrices.

\begin{definition}[$(\epsilon,\beta)$-representative, Definition B.2 in~\cite{cn20}]\label{def:eps_representative}
   A set of matrices $\{\Pi_{j} \in \R^{m \times k}\}_{j=1}^{L}$ is said to be $(\epsilon,\beta)$-representative %(Definition~\ref{def:eps_representative}) 
   for any $\epsilon, \beta \in (0,1/2)$ if
\begin{align*}
  \forall\ \|v\|_2=1,\ \sum_{j=1}^{L} {\bf 1}\{  (1 - \epsilon) \leq \|\Pi_{j} v\|_2 \leq (1 + \epsilon)\}  \geq (1-\beta) L .  %%
\end{align*}
\end{definition}
The above definition implies that if a set of matrices $\{\Pi_{j} \in \R^{m \times k}\}_{j=1}^{L}$ satisfies this property, for any any unit vector $v$, most of the projections $\Pi_j v$, approximately preserve its length. 

% \section{Probability Tools}\label{sec:tools}

We will make use of Hoeffding’s Inequality:
\begin{lemma}[Hoeffding’s Inequality~\cite{h63}]\label{thm:blm13_hoeffding_inequality}
   Let $X_{1}, \ldots, X_{n}$ be independent random variables such that $X_{i} \in\left[a_{i}, b_{i}\right]$ almost surely for $i \in[n]$ and let $S=\sum_{i=1}^{n} X_{i}-\E[X_{i}]$. Then, for every $t>0$ :
\begin{align*}
   \Pr[S \geq t] \leq \exp \left(-\frac{2 t^{2}}{\sum_{i=1}^{n}\left(b_{i}-a_{i}\right)^{2}}\right) .
\end{align*}
\end{lemma}

We will use the following guarantee (restated in our notation) from \cite{cn20},
\begin{lemma}\label{lem:epsilon_representative}
Let $\beta \in [ 0.1,0.2]$. For any accuracy parameter $\epsilon \in (0,1)$ and failure probability $\delta \in (0,1)$, if $m = \Omega(\frac{1}{\epsilon^2}), L = \Omega((d + \log(1/\delta)) \log(d/\epsilon))$, the set of sketching matrices $\{\Pi_{j} \in \R^{m \times k}\}_{j=1}^{L}$ where every entry of each matrix is drawn i.i.d. from $\N(0,1/m)$ satisfies:
\begin{align*}
    \forall\|v\|_2 = 1,\ \sum_{j=1}^{L} {\bf 1}\{ (1 - \epsilon) \leq \|\Pi_{j} v\|_2 \leq (1 + \epsilon)\} \geq (1-\beta) L
\end{align*}
with probability at least $1-\delta$.
\end{lemma}

The above lemma implies that with high probability, the sketch matrices $\{\Pi_{j}\}_{j=1}^{L}$ is $(\epsilon, \beta)$-representative.

\subsection{Applications of our Data-Structure}\label{sec:application}
Developing efficient algorithms for Nearest Neighbor Search (NNS) has been a focus of intense research, driven by a variety of real-world applications ranging from computer
vision, information retrieval to database query~\cite{bm01, diim04, sdi08}. A line of work, starting with the foundational results of~\cite{im98, kor00}, have obtained sub-linear query time in $n$ for the approximate variant in order to solve the problem that fast algorithms for exact NNS~\cite{cla88, mei93} consumes impractical space complexity. The Locality Sensitive Hashing (LSH)
approach of~\cite{im98} gives a Monte Carlo randomized approach with low memory and query time, but it does not support adaptive queries. Although the algorithms of~\cite{kor00, kle97} support adaptive queries and have sublinear query time, they use large space and are only designed for finding the approximate single nearest neighbor and do not provide distance estimates to all data points in the database per query.

For nonparametric models including SVMs, distance estimates
to potentially every point in the dataset may be required~\cite{ alt92, ams97, hss08}. For
simpler tasks like $k$-nearest neighbor classification or database search, modifying
previous approaches to return $k$ nearest neighbors instead of $1$, results in a factor $k$ increase in overall query time. Therefore, developing efficient deployable variants in adversarial settings is an important endeavor, and an adaptive approximate distance estimation procedure is a useful primitive for this purpose.

%% file: tech.tex
\section{Technique Overview}
%\vspace{-2mm}
 
For a set of points $\mathcal{X} = \{x_1, x_2, \dots, x_n\} \in \R^d$, a Mahalanobis distance metric $d$ on $\X$ is characterized by a positive semidefinite matrix $A \in \R^{d \times d}$ such that 
\begin{align*}
d(x_i, x_j) = \sqrt{(x_i - x_j)^\top A (x_i - x_j)}.
\end{align*}
First we provide a randomized sketching data structure $\mathcal{D}$ which can answer approximate non-adaptive Mahalanobis distance estimation queries  $q \in \R^d$ by  $ \tilde{d}_{i} = \|\Pi Uq-\tilde{x}_{i}\|_2$ in Theorem~\ref{thm:jl_guarantees}, where $\Pi$ is a Johnson-Lindenstrauss sketching~\cite{jl84} matrix  and $U^{\top} U = A$.

To make our data structure resistant to adaptively chosen queries,
we will use the $(\epsilon, \beta)$-representative in Definition~\ref{def:eps_representative}
where A set of matrices $\{\Pi_{j} \in \R^{m \times k}\}_{j=1}^{L}$ is said to be $(\epsilon,\beta)$-representative for any $\epsilon, \beta \in (0,1/2)$ if at least $(1-\beta) L$ sketching matrice can achieve $(1 \pm \epsilon)$-approximate distance estimation.
With 
\begin{align*}
L = O((d+ \log \frac{1}{\delta}) \log(d/\epsilon))
\end{align*}
independent instantiation of randomized sketching data structure $\mathcal{D}$ satisfying the  $(\epsilon, 0.1)$-representative property, 
we know that for any given query vector $q \in \R^d$, most of the sketching matrices approximately preserves its length. 
Then we design a set of unit vectors $\{v_i = \frac{U(q-x_i)}{\|U(q-x_i)\|_2}\}_{i=1}^n$ for query $q \in \R^d$ and data points $\{x_1, x_2, \dots, x_n\}\subset \R^d$ and show that most of the projections $\Pi_j$ satisfy $(1 \pm \epsilon)$-approximate Mahalanobis distance estimation.
We define a set $\mathcal{J}$ as $\mathcal{J} \coloneqq \{j:  (1 - \epsilon) \leq \|\Pi_{j} U v\| \leq (1 + \epsilon)\}$ which has size at least $0.9L$. 
Now query using any randomized sketching matrix $\Pi_j$ gives us $\epsilon$-approximation with constant $0.9$ success probability. 
To boost the success probability from constant probability to high probability,
we sample $R = O(\log(n/\delta))$ indexes $\{j_r\}_{r=1}^{R}$ from the set $\mathcal{J}$, and obtain $R$ estimates of the Mahalanobis distance between $q$ and $x_i$.
By Hoeffding's Inequality (Lemma~\ref{thm:blm13_hoeffding_inequality}), we know that the median of the sampled $R$ distance estimates can provide $(1 \pm \epsilon)$-approximate query with high probability. 

To support sparse update for the Mahalanobis metric matrix $U$, we update $\tilde{x}_{i,j}$ with the sparse update matrix $B$ for $n$ data points and $L$ sketching matrice respectively in $O((m+d)nL)$ time.
To update the $i$-th data point with a new vector $z \in \R^d$, we need to recompute  $\tilde{x}_{i,j} = \Pi_j U z$ for $L$ sketching matrice respectively in $O((m+d)kL)$ time.

To support sampling an index $i$ with probability proportional to the distance between $q$ and $x_i$, we first leverage a segment tree to obtain $\sum_{\ell = i}^{j} x_{\ell}$ in logarithmic time, and compute $\sum_{\ell = i}^{j} \| q - x_{\ell}\|_{A}$. In each iteration, we determine left or right half to proceed by sampling a random value between $[0,1]$ uniformly and compare with $\frac{\sum_{\ell = l}^{m} \| q - x_{\ell}\|_{A}}{\sum_{\ell = l}^{r} \| q - x_{\ell}\|_{A}}$. After $T = O(\log n)$ iterations, we can obtain the final sampled index in $O(\log^2 n + kd \log n)$ time.

%% file: sketching.tex
\section{JL sketch for approximate Mahalanobis distance estimation}~\label{sec:sketch}
In this section, we provide a data structure which uses the well studied Johnson-Lindenstrauss sketch~\cite{jl84} to solve the approximate Mahalanobis distance estimation problem. We remark that other standard sketches like the AMS sketch~\cite{ams99} or \textsc{CountSketch}~\cite{ccf02} can also be used to provide similar data-structures which can support Mahalanobis distance estimation with Monte Carlo guarantees.

%\subsection{JL sketch for approximate Mahalanobis distance estimation}
\begin{algorithm*}[h]\caption{Informal version of Algorithm~\ref{alg:jl_guarantees_full}. JL sketch for approximate Mahalanobis distance estimation}
\label{alg:jl_guarantees}
\begin{algorithmic}[1]
\State {\bf data structure} {\sc JLMonCarMaintenance} \Comment{Theorem~\ref{thm:jl_guarantees}}
\State {\bf members}
\State \hspace{4mm} $d, k, m, n \in \mathbb{N}_{+}$ \Comment{$d$ is the dimension of data, $m$ is the sketch size, $n$ is the number of points}
\State \hspace{4mm} $U \in \R^{k \times d}$
\State \hspace{4mm} $\{\tilde{x}_{i}\} \in \R^{m}$ for $i \in [n]$ \Comment{Sketch of the $n$ points}
\State {\bf end members}
\State
\Procedure{Initialize}{$U \in \R^{k \times d}, \{x_i\}_{i\in [n]} \subset \R^d, \epsilon \in (0,1), \delta \in (0,1)$} 
    \State $m \gets  \Omega(\frac{\log n }{\epsilon^2})$  \Comment{Initialize the sketch size}
    \State $U \gets U$
    \State For $i \in [n]$, $x_i \gets x_i$
    \State Let $\Pi \in \R^{m \times k}$ with entries drawn iid from $\mathcal{N}(0,1/m)$. \Comment{AMS or CountSketch also valid}
    \State For $i \in [n]$, $\tilde{x}_{i} \gets \Pi Ux_i$  \label{alg:jl_sketch_init_matrix_mul:informal}
\EndProcedure
\State
\Procedure{Query}{$q \in \R^d$}%\label{alg:query_all_procedure}
    \State For $i \in [n]$, let $ \tilde{d}_{i} \gets \|\Pi Uq-\tilde{x}_{i}\|_2$\label{alg:jl_sketch_query_matrix_mul:informal}
    \State \Return $\{\tilde{d}_i\}_{i=1}^n$
\EndProcedure
% \State
\State {\bf end data structure}
\end{algorithmic}
\end{algorithm*}
% \vspace{-4mm}

\begin{theorem}[Guarantees for JL sketch for Mahalanobis Metrics]\label{thm:jl_guarantees}
There is a data structure (Algorithm~\ref{alg:jl_guarantees}) for Approximate Mahalanobis Distance Estimation with the following procedures:
\begin{itemize}
    \item \textsc{Initialize}$(U \in \R^{k \times d}, \{x_1, x_2, \dots, x_n\}\subset \R^d, \epsilon \in (0,1), \delta \in (0,1))$: Given a matrix $U \in \R^{k \times d}$, a list of vectors $\{x_1, x_2, \cdots, x_n\} \subset \R^d$, accuracy parameter $\epsilon \in (0,1)$, and failure probability $\delta \in (0,1)$, the data-structure pre-processes in time $\wt O((d+m)k)$.
    \item \textsc{Query}$(q \in \R^d)$: Given a query vector $q \in \R^d$, Query outputs approximate Mahalanobis distance estimates $\{d_i\}_{i=1}^n$ in time $O((d+m)k)$ such that 
    \begin{align*}
         & ~ \Pr \Big[ (1 - \epsilon)\|q-x_i\|_A \leq   d_i \leq (1 + \epsilon)\|q-x_i\|_A  ,  \forall i \Big] \\ 
         & ~ \geq 1 - \delta .
    \end{align*}
    where $A = U^\top U \in \R^{d \times d}$ is a positive semidefinite matrix which characterizes the Mahalanobis distances. Note that the above $\forall i$ indicates for all $i \in [n]$.
\end{itemize}
\end{theorem}
% We delay the proof to Appendix~\ref{sec:sketch_app}.

\begin{proof}

The \textsc{Initialize} pre-processing time complexity is dominated by line~\ref{alg:jl_sketch_init_matrix_mul} which takes $O((d+m)k)$.
The \textsc{Query} time complexity is dominated by line~\ref{alg:jl_sketch_query_matrix_mul} which takes $O((d+m)k)$.

\iffalse
{\color{red}
Let $v_{i}=\frac{q-x_{i}}{\|q-x_{i}\|_{A}}$, for a query point $q$ and a dataset point $x_{i}$. We need to bound :
\begin{align*}
      \Pr[|\|\Pi U v\|^{2}-1| \geq \epsilon] \leq ??? \leq \delta
\end{align*}
here
} \Zhao{Lianke can't finish this, aravind you should take care of it.} \Aravind{Ok}
\fi

{\bf Proof of correctness for Query:} We will use the following lemma which is a standard result used in proofs of the Johnson-Lindenstrauss lemma.

\begin{lemma}[Distributional JL (DJL) lemma]\label{lem:dist_jl}
For any integer $k \geq 1$ and any $0 < \epsilon, \delta < 1/2$, there exists a distribution $D_{\epsilon,\delta}$ over $m \times k$ real matrices for some $m \leq c\epsilon^{-2} \log (1/\delta)$ for some constant $c > 0$ such that:
\begin{align*}
     \forall\ \|u\|_2 = 1, \Pr_{\Pi \sim D_{\epsilon,\delta}} [(1-\epsilon) \leq \|\Pi u\|_2 \leq (1 + \epsilon)] \geq 1 - \delta .
\end{align*}
\end{lemma}
We will show that our choice of $\Pi \in \R^{m \times k}$ with entries drawn iid from $\N(0,1/m)$ satisfies the above guarantee. 

Let $u = (u_1,u_2, \dots, u_k)$. Since $\Pi_{i,j} \sim \N(0,1/m)$ for all $i \in [m], j \in [k]$, we have that each coordinate $i \in [m]$ of $(\Pi u)_i = \sum_{j=1}^k \Pi_{i,j} u_j$ is the sum of $k$ independent normally distributed random variables. Therefore, $\{(\Pi u)_i\}_{i=1}^m$ are are all normally distributed real variable with their mean and variance being the sums of the individual means and variances. Note that $\E[(\Pi u)_i] = \sum_{j=1}^k u_j \E[\Pi_{i,j}] = 0$ and $\Var[(\Pi u)_i] = \sum_{j=1}^k u_j^2 \Var[\Pi_{i,j}] = (1/m) \cdot \|u\|_2^2 = 1/m$. Therefore, $(\Pi u)_i \sim \N(0, 1/m)$. A random vector in $\R^m$ with all its coordinates distributed as $\N(0,\sigma^2)$ can be viewed as being drawn from a $m$-dimensional spherical Gaussian $\N(\mathbf{0},\sigma^2 I_{m \times m})$. Therefore, we have $(\Pi u) \sim \N(\mathbf{0},\frac{1}{m} I_{m \times m})$. We will now show how to use the Gaussian Annulus Theorem \cite[Theorem 2.9]{bhk20} to finish our proof.

\begin{lemma}[Gaussian Annulus Theorem]
For any $x \sim \N(\mathbf{0},I_{m \times m}), \beta \leq \sqrt{m}$, for some constant $c > 0$, we have that:
\begin{align*}
 \Pr \Big[(\sqrt{m} - \beta) \leq \|x\|_2 \leq (\sqrt{m} + \beta ) \Big] \geq 1 - 3^{-c\beta^2} .
\end{align*}
\end{lemma}

We have that $(\Pi u) \sim \N(\mathbf{0},\frac{1}{m} I_{k \times k})$. Therefore, we have $(\Pi u) = \frac{1}{\sqrt{m}} x$ for $x \sim \N(\mathbf{0},I_{k \times k})$. The condition that $(1 - \epsilon) \leq  \|\Pi u\|_2 \leq (1 + \epsilon)$ is equivalent to $(\sqrt{m} - \epsilon \sqrt{m}) \leq \|x\|_2 \leq (\sqrt{m} + \epsilon \sqrt{m})$. Therefore, we have that, for some constant $c > 0$,
\begin{align*}
  \Pr_{\Pi \sim D_{\epsilon,\delta}} \Big[ (1 - \epsilon) \leq  \|\Pi u\|_2 \leq (1 + \epsilon) \Big] \geq 1 - 3^{-cm\epsilon^2}.
\end{align*}

To prove the correctness for our Query operation, we will apply a union bound over the choices of $\{u_i = \frac{U(q-x_i)}{\|U(q-x_i)\|_2}\}_{i=1}^n$ in the above lemma. The failure probability is upper bounded by $n \cdot 3^{-cm\epsilon^2}$. To have this be at most delta, we need to have $m \geq \frac{c'}{\epsilon^2}\log{(n/\delta)}$ for some constant $c'$. In particular, if we want a with high probability guarantee, taking $m = \Omega(\frac{\log n}{\epsilon^2})$ should suffice.

Thus, we complete the proof.
\end{proof}

\iffalse
Given a unit vector $v \in \R^d$ i.e. $\|v\|_2 = 1$ and $k, \epsilon$ as above, with probability $\geq 1 - \dfrac{1}{n^4}$,
\begin{align*}
     (1-\epsilon)\cdot k \leq \|Mv\|_2^2 \leq (1+\epsilon)\cdot k
\end{align*}

\begin{proof}
    Let $Mv = z$ where $z \in \R^k$ and the $i$'th co-ordinate of $z$ i.e. $z(i) = \langle u_i, v\rangle$ where $u_i \sim N(\mathbf{0}, I_{d \times d})$. From the previous lecture, we know that $\langle u_i, v\rangle \sim N(0,\|v\|_2^2) = N(0,1)$. Thus, $z = (z_1, z_2, \dots, z_k)$ is a random $k$-dimensional Gaussian vector i.e. $z \sim N(\mathbf{0}, I_{k \times k})$. Applying Proposition  with $d = k$, we have 
    \begin{align*}
       \Pr\left[ \left| \|z\|^2_2 - k \right| > t \sqrt{k} \right] \le 3 e^{-t^2 / 6}   
    \end{align*}

    Taking $t = \epsilon\sqrt{k}$, we have
    \begin{align*}
       \Pr\left[ \left| \|z\|^2_2 - k \right| > \epsilon \cdot k \right] &\le 3 e^{-\epsilon^2 k / 6}
       \\&\leq 3 e^{-\epsilon^2 \cdot  \frac{c\log(n)}{\epsilon^2 \times 6 }}\hspace{20pt} (\text{since } k \geq \dfrac{c \log n}{\epsilon^2})
       \\&\leq 3n^{-{\frac{c}{6}}}
    \end{align*}
    Taking $c \gtrsim 24$ gives us the required result.
    Thus, we complete the proof.
\end{proof}
\fi

%% file: alg.tex
\section{Online Adaptive Mahalanobis Distance Maintenance}~\label{sec:adaptive_maintaince}
Now, we move to the data structure and the algorithm to solve the online version of the problem and the corresponding proofs. Our main result in this section is the following:

\begin{algorithm*}[!ht]\caption{Informal version of Algorithm~\ref{alg:mahalanobis_maintenance_full}. Mahalanobis Pseudo-Metric Maintenance: members, initialize and update}
\label{alg:mahalanobis_maintenance}
\begin{algorithmic}[1]
\State {\bf data structure} {\sc MetricMaintenance} \Comment{Theorem~\ref{thm:main}}
\State {\bf members}
\State \hspace{4mm} $L, m, k$\Comment{$L$ is the number of sketches, $m$ is the sketch size }
\State \hspace{4mm} $ d, n\in \mathbb{N}_{+}$ \Comment{ $n$ is the number of points, $d$ is dimension} 
\State \hspace{4mm} $U \in \R^{k \times d}$
\State \hspace{4mm} $x_i \in \R^d$ for $i \in [n]$
\State \hspace{4mm} $\{\tilde{x}_{i,j}\} \in \R^{m}$ for $i \in [n], j \in [L]$ \Comment{The sketch of data points}
\State \hspace{4mm} $\textsc{Tree}$  $\textsc{tree}$ \Comment{Segment tree(Alg~\ref{alg:segment_tree}) $\textsc{tree}.\textsc{Query}(i, j)$ to obtain $\sum_{\ell = i}^{j} x_{\ell}$}
\State {\bf end members}
\State
\Procedure{Initialize}{$U \in \R^{k \times d}, \{x_i\}_{i \in [n]} \subset \R^d, \epsilon \in (0,1), \delta \in (0,1)$} \Comment{Lemma~\ref{lem:init}}
    \State $m \gets O(\frac{1}{\epsilon^2})$  \Comment{Initialize the sketch size}
    \State $L \gets O((d+ \log \frac{1}{\delta}) \log(d/\epsilon))$ \Comment{Initialize the number of copies of sketches}
    \State $U \gets U$
    \State For $i \in [n]$, $x_i \gets x_i$ \Comment{Initialize the data points from input} %\label{alg:init_x_assign_start}
    \State For $j \in [L]$, let $\Pi_j \in \R^{m \times k}$ with entries drawn iid from $\mathcal{N}(0,1/m)$. %\label{alg:init_pi_drawn} %\Zhao{In this intialize function, the $l$ appears in 3 places, please change all of them to $L$}
    \State For {$i \in [n]$}, For {$j \in [L]$},  $\tilde{x}_{i,j} \gets \Pi_jUx_i$ %\label{alg:init_assign_tilde_x}
    \State $\textsc{tree}.\textsc{Init}(\{x_i\}_{i \in [n]}, \textsc{tree}.t, 1, n)$ \Comment{Initialize the segment tree using all data points.}
    % \State \lianke{add segment tree init for $\{x_i\}_{i=1}^{n}$}
\EndProcedure
\State
\Procedure{UpdateU}{$u \in \R^d, a \in [k]$} \Comment{We can consider sparse update for $U$, Lemma~\ref{lem:update}}.
    \State $B \gets \{0\}_{k \times d}$
    \State $B_{a} \gets u^\top$ \Comment{$B_a$ denotes the $a$'th row of $B$.} %\label{alg:update_b_a}
    \State $U \gets U + B$ %\label{alg:update_U_add_B}
    \State For {$i \in [n]$},  For {$j \in [L]$},  $\tilde{x}_{i,j} \gets \tilde{x}_{i,j} + \Pi_jBx_i$ %\label{alg:update_x_i_j}
\EndProcedure
\State
\Procedure{UpdateX}{$z \in \R^d, i \in [n]$} \Comment{ Lemma~\ref{lem:updateX}}.
    \State $x_i \gets z$
    \State  For {$j \in [L]$}, $\tilde{x}_{i,j} \gets \Pi_j U z$ \label{alg:updateX_start_loop:informal}
    \State $\textsc{tree}.\textsc{Update}(\textsc{tree}.t, i, z)$
\EndProcedure
\State {\bf end data structure}
\end{algorithmic}
\end{algorithm*}

\begin{algorithm*}[h]\caption{Informal version of Algorithm~\ref{alg:query_full}. Online Adaptive Mahalanobis Distance Maintenance: queries.}\label{alg:query}
\begin{algorithmic}[1]
\State {\bf data structure} {\sc MetricMaintenance} \Comment{Theorem~\ref{thm:main}}
% \Procedure{Query}{$q \in \R^d, i \in [n]$}\label{alg_query_procedure} \Comment{Lemma~\ref{lem:query}}
%     \State $R \gets O(\log(n/\delta))$ \Comment{Number of sampled sketches}
%     \State Sample $j_1, j_2, \dots, j_R$ i.i.d. with replacement from $[L]$.\label{alg:query_sample_j}
%     \State For $r \in [R]$, $d_{i,r} \gets \|\Pi_{j_r}Uq- \tilde{x}_{i,r}\|_2$\label{alg:query_norm}
%     % \State $\tilde{d}_i \gets \Median(\{d_{i,r}\}_{r=1}^R)$\label{alg:query_median} 
%     \State $\tilde{d}_i \gets \Median_{r\in [R]} \{ d_{i,r} \} $\label{alg:query_median} 
%     %\lianke{should be $\Median(\{y_{i,k}\}_{k=1}^r)$?}\Aravind{Yes, that's right.}
%     \State \Return $\tilde{d}_i$
% \EndProcedure
% \State
\Procedure{QueryPair}{$i,j \in [n]$}\label{alg:query_pair:informal}\Comment{Lemma~\ref{lem:query_pair}}
    \State $R \gets O(\log(n/\delta))$ \Comment{Number of sampled sketches}
    \State For $r \in [R]$,  $p_r \gets \| \tilde{x}_{i,r} - \tilde{x}_{j,r} \|_2$ %\label{alg:query_pair_for_loop} 
    \State $p \gets \Median_{r \in [R]}\{ p_r\}$ %\label{alg:query_pair_median}
    \State \Return $p$
\EndProcedure
\State
\Procedure{QueryAll}{$q \in \R^d$}\label{alg:query_all_procedure:informal} \Comment{Lemma~\ref{lem:query_all}}
    \State $R \gets O(\log(n/\delta))$  \Comment{Number of sampled sketches}%\Zhao{Change small $r$ to Big $R$}
    \State Sample $j_1, j_2, \dots, j_R$ i.i.d. with replacement from $[L]$.\label{alg:query_all_sample_j:informal} %\Zhao{Change $r$ to $R$, change $l$ to $L$}
    \State For $i \in [n], r \in [R]$, let $ d_{i,r} \gets \|\Pi_{j_r}Uq-\tilde{x}_{i,r}\|_2$%\label{alg:query_all_norm} %\Zhao{Change $k \in [r]$ to $r \in [R]$}
    \State For $i \in [n], \tilde{d}_i \gets \Median_{r\in [R]} \{ d_{i,r} \}$%\label{alg:query_all_median}%\lianke{should be $\Median(\{y_{i,k}\}_{k=1}^r)$?}%Aravind{Yes}
    \State \Return $\{\tilde{d}_i\}_{i=1}^n$
\EndProcedure
\State {\bf end data structure}
\end{algorithmic}
\end{algorithm*}

\begin{theorem}[Main result]\label{thm:main}
There is a data structure (Algorithm~\ref{alg:mahalanobis_maintenance} and \ref{alg:query}) for the Online Approximate Adaptive Mahalanobis Distance Estimation Problem (Definition ~\ref{def:main_problem}) with the following procedures:
\begin{itemize}
    \item \textsc{Initialize}$(U \in \R^{k \times d}, \{x_1, x_2, \dots, x_n\}\subset \R^d, \epsilon \in (0,1), \delta \in (0,1))$: Given a matrix $U \in \R^{k \times d}$, data points $\{x_1, x_2, \dots, x_n\}\subset \R^d$, an accuracy parameter $\epsilon$ and a failure probability $\delta$ as input, the data structure preprocesses in time $O((m+d)knL)$.
    \item \textsc{UpdateU}$(u \in \R^d, a \in [k])$: Given an update vector $u \in \R^d$ and row index $a \in [k]$, the data structure takes $u$ and $a$ as input and updates the $a$'th row of $U$ by adding $a$ to it i.e. $U_{a,:} \gets U_{a,:} + u$ in time $O((m+d)nL)$.
    \item \textsc{UpdateX}$(z \in \R^{d}, i \in [n])$: Given an update vector $z \in \R^d$ and index $i \in [n]$, the \textsc{UpdateX} takes $z$ and $i$ as input and updates the data structure with the new $i$-th data point in $O((m+d)kL + \log n)$ time.
    % \item \textsc{Query}$(q \in \R^d, i \in [n])$. Given a query query point $q \in \R^d$ and index $i \in [n]$, the \textsc{Query} operation takes $q$ and $i$ as input and approximately estimates the Mahalanobis distances from $q$ to data point $x_i$ in time $O((m+d)kR)$ and output $\tilde{d}_i$ such that:
    % \begin{align*}
    %     (1-\epsilon)\|q-x_{i}\|_{A} \leq \tilde{d}_{i} \leq(1+\epsilon)\|q-x_{i}\|_{A}
    % \end{align*}
    % with probability at least $1 -\delta$.
    %\lianke{we need to unify the high probability estimate equation style. There are two styles: one from the original paper; one from Aravind}
    \item \textsc{QueryPair}$(i, j \in [n])$. Given two indexes $i, j \in [n]$, the \textsc{QueryPair} takes $i, j$ as input and approximately estimates the Mahalanobis distance from $x_i$ to $x_j$ in time $O(mR)$ and output a number $p$ such that:
    %\begin{align*}
    $
        (1-\epsilon) \| x_i - x_j \|_A \leq p \leq (1+\epsilon) \cdot \| x_i - x_j \|_A
    $
    %\end{align*} 
    with probability at least $1 -\delta$.
    \item \textsc{QueryAll}$(q \in \R^d)$: Given a query point $q \in \R^d$, the \textsc{QueryAll} operation takes $q$ as input and approximately estimates the Mahalanobis distances from $q$ to all the data points $\{x_1, x_2, \dots, x_n\}\subset \R^d$ in time $O((m+d)knR)$ i.e. it provides a set of estimates $\{\tilde{d}_i\}_{i=1}^n$ such that:
    \begin{align*}
       \forall i \in[n], (1-\epsilon)\|q-x_{i}\|_{A} \leq \tilde{d}_{i} \leq(1+\epsilon)\|q-x_{i}\|_{A} 
    \end{align*}
     with probability at least $1 -\delta$, even for a sequence of adaptively chosen queries.
    % \item \textsc{SubSum}$(q \in \R^d, i \in [n], j \in [n] )$. Given a query point $q \in \R^{d}$ and two indexes $i, j \in [n]$ as input, \textsc{SubSum} outputs $\sum_{l=i}^{j} \| q - x_l \|_A$ in $O(\log n + kd)$ time.
    \item \textsc{SampleExact}$(q \in \R^d )$. Given a query point $q \in \R^{d}$ as input,  \textsc{SampleExact} samples an index $i \in [n]$ with probability $d_i / \sum_{j=1}^n d_j$ in $O(\log^2 n + kd \log n)$ time.
\end{itemize}
\end{theorem}
\begin{proof}
We complete the proof for Theorem~\ref{thm:main} using the following Lemma~\ref{lem:init}, Lemma~\ref{lem:update}, Lemma~\ref{lem:updateX}, 
%Lemma~\ref{lem:query}, 
Lemma~\ref{lem:query_pair}, Lemma~\ref{lem:query_all} and Lemma~\ref{lem:sample_exact}.
%We complete the proof for Theorem~\ref{thm:main} using the following Lemmas ~\ref{lem:init}, Lemma~\ref{lem:update}, Lemma~\ref{lem:query}, Lemma~\ref{lem:query_pair} and Lemma~\ref{lem:query_all}, .
\end{proof}

\begin{lemma}[Initialization Time]\label{lem:init}
Given a matrix $U \in \R^{k \times d}$, data points $\{x_1, x_2, \dots, x_n\}\subset \R^d$, an accuracy parameter $\epsilon$ and a failure probability $\delta$ as input, the time complexity of \textsc{Initialize} function is $O((m+d)knL)$.
\end{lemma}
% The proof is delayed to Appendix~\ref{sec:proof_time_init}.
\begin{proof}
 The initialization part has three steps:
\begin{itemize}
    \item Line~\ref{alg:init_x_assign_start} in Algorithm~\ref{alg:mahalanobis_maintenance_full} takes $O(nd)$ time to assign $x_i \in \R^d$ for $n$ times.
    \item Line~\ref{alg:init_pi_drawn} in Algorithm~\ref{alg:mahalanobis_maintenance_full} takes $O(mk)$ time to iid sample from $\mathcal{N}(0,1/m)$.
    \item Line~\ref{alg:init_assign_tilde_x} in Algorithm~\ref{alg:mahalanobis_maintenance_full} takes $O((m+d)knL)$ time to execute $n \times L$ times matrix multiplications.
    \item The initialization of segment tree needs $O(nd)$ time.
\end{itemize}
Therefore, the total time for initialization is
    \begin{align*}
    & ~ O(nd) + O(mk) + O((m+d)knL) + O(nd) \\
    = & ~O((m+d)knL).
\end{align*}
Thus, we complete the proof.
\end{proof}

\begin{lemma}[UpdateU Time]\label{lem:update}
 Given an update vector $u \in \R^d$ and row index $a \in [k]$, the \textsc{UpdateU} takes $u$ and $a$ as input and updates the $a$'th row of $U$ by adding $a$ to it i.e. $U_{a,:} \gets U_{a,:} + u$ in $O((m+d)nL)$ time.
\end{lemma}
% The proof is delayed to Appendix~\ref{sec:proof_time_update}.
\begin{proof}
% {\bf Proof of \textsc{UpdateU}}.
The update part has three steps:
\begin{itemize}
    \item Line~\ref{alg:update_b_a} in Algorithm~\ref{alg:mahalanobis_maintenance_full} takes $O(d)$ time to assign $u^{\top}$ to $B_a$.
    \item Line~\ref{alg:update_U_add_B} in Algorithm~\ref{alg:mahalanobis_maintenance_full} takes $O(d)$ time to update $U$ because $B$ is sparse and we can ignore $k-1$ rows containing all zero elements.
    \item Line~\ref{alg:update_x_i_j} in Algorithm~\ref{alg:mahalanobis_maintenance_full}  take $O((m+d)nL)$ time to compute sparse matrix multiplication for $n \times L$ times and each sparse matrix multiplication takes $O(m+d)$ time because $B$ is sparse and we can ignore the $k-1$ rows containing all zeros elements. 
\end{itemize}
Therefore, the total time for update is
    \begin{align*}
     &~O(d) + O(d) + O((m+d)nL) \\
    =  &~ O((m+d)nL) .
\end{align*}
Thus, we complete the proof.
\end{proof}

\begin{lemma}[UpdateX Time]\label{lem:updateX}
 Given a new data point $z \in \R^d$ and index $i \in [n]$, the \textsc{UpdateX} takes $z$ and $i$ as input and updates the data structure with the new $i$th data point in $O((m+d)kL + \log n)$ time.
\end{lemma}
% The proof is delayed to Appendix~\ref{sec:proof_time_updatex}.
\begin{proof}
% {\bf Proof of \textsc{UpdateU}}.
The \textsc{UpdateX} operation has one step:
\begin{itemize}
    \item Line~\ref{alg:updateX_start_loop} in Algorithm~\ref{alg:mahalanobis_maintenance_full} takes $O((m+d)kL)$ time to compute $\Pi_j \cdot (U \cdot z)$ for $L$ times.
    \item The \textsc{Tree.Update} operation takes $O(\log n)$ time to complete.
\end{itemize}
Therefore, the total time for \textsc{UpdateX} is $O((m+d)kL + \log n)$.
Thus, we complete the proof.
\end{proof}

\begin{lemma}[QueryPair]\label{lem:query_pair}
Given two indexes $i, j \in [n]$, the \textsc{QueryPair} takes $i, j$ as input and approximately estimates the Mahalanobis distance from $x_i$ to $x_j$ in time $O(mR)$ and output a number $p$ such that:
%\begin{align*}
$
    (1-\epsilon)\|x_{i}-x_{j}\|_{A} \leq p \leq(1+\epsilon)\|x_{i}-x_{j}\|_{A}.
$
%\end{align*}
with probability at least $1 -\delta$.
\end{lemma}
% We defer the proof to Appendix~\ref{sec:proof_lemma_query_pair}.

\begin{proof}
{\bf Proof of Running Time.} 

We can view the \textsc{QueryPair} operation as having the following two steps:

\begin{itemize}
    \item The for-loop in line~\ref{alg:query_pair_for_loop} takes $O(mR)$ time because $\|\tilde{x}_{i,r} - \tilde{x}_{j,r}\|_2$ takes $O(m)$ time and it is executed for $R$ times.
    \item Line~\ref{alg:query_pair_median} takes $O(R)$ time to find median value from $\{p_{r}\}_{r=1}^R$.
\end{itemize}
Thus, the total running time of $\textsc{QueryPair}$ is
\begin{align*}
 ~ O(R) + O(mR) 
    = ~ O(mR).
\end{align*}

{\bf Proof of Correctness.} 

Instead of choosing $v=\frac{U(q-x_{i})}{\|q-x_{i}\|_{A}}$ in the proof of correctness for \textsc{Query}(Lemma~\ref{lem:query_all}), we should choose $v=\frac{U(x_i-x_{j})}{\|x_i-x_{j}\|_{A}}$ for \textsc{QueryPair}$(i,j)$. 

Accordingly, we also consider
\begin{align*}
z_{i,j} \coloneqq \Median_{r \in [R]}\{\tilde{y}_{i,j,r}\}
\end{align*}
and 
\begin{align*}
\tilde{y}_{i,j,r}:=\|\Pi_{j_{r}}U v\|.
\end{align*}
This gives us:
\begin{align*}
   (1-\epsilon)\|x_i-x_{j}\|_{A} \leq p \leq(1+\epsilon)\|x_i-x_{j}\|_{A}. 
\end{align*}
with probability at least $1 -\frac{\delta}{n}$. 

This concludes the proof of correctness for the output of \textsc{QueryPair}.
\end{proof}

\begin{lemma}[QueryAll]\label{lem:query_all}
Given a query point $q \in \R^d$, the \textsc{QueryAll} operation takes $q$ as input and approximately estimates the Mahalanobis distances from $q$ to all the data points $\{x_1, x_2, \dots, x_n\} $ $ \subset \R^d$ in time $O((m+d)knR)$ i.e. it provides a set of estimates $\{\tilde{d}_i\}_{i=1}^n$ such that:
\begin{align*}
   \forall i \in[n], (1-\epsilon)\|q-x_{i}\|_{A} \leq \tilde{d}_{i} \leq(1+\epsilon)\|q-x_{i}\|_{A}. 
\end{align*}
with probability at least $1 -\delta$.
\end{lemma}

\begin{proof}
% {\bf Proof of \textsc{QueryAll} time}.

{\bf Proof of running time.} 

We can view the \textsc{QueryAll} operation as having the following three steps:
\begin{itemize}
    \item Line~\ref{alg:query_all_sample_j} takes $O(R)$ time to sample $\{j_r\}_{r=1}^{R}$ from $[L]$.
    \item Line~\ref{alg:query_all_norm} takes $O((m+d)knR)$ time because $\|\Pi_{j_r}Uq-\tilde{x}_{i,r}\|_2$ takes $O((m+d)k)$ time and it is executed for $n \times R$ times.
    \item Line~\ref{alg:query_all_median} takes $O(nR)$ time to find median value from $\{d_{i,r}\}_{r=1}^R$ for $n$ times.
\end{itemize}
Thus, the total running time of $\textsc{QueryAll}$ is
\begin{align*}
& ~ O(R) + O((m+d)knR) + O(nR) \\
  = & ~   O((m+d)knR).
\end{align*}

% \Danyang{Lianke, why is there a style difference between proof. for running time and proof of correctness?}\lianke{I have fixed this typo.}
{\bf Proof of Correctness.}
\iffalse
From the proof of correctness for \textsc{Query} (Lemma~\ref{lem:query}), we know that for any given $i \in [n]$, we have that:
\begin{align*}
   (1-\epsilon)\|q-x_{i}\|_{A} \leq \tilde{d}_{i} \leq(1+\epsilon)\|q-x_{i}\|_{A}. 
\end{align*}
with probability at least $1 -\frac{\delta}{n}$. 
\fi 

We will use the $(\epsilon,\beta=0.1)$-representative property (See Definition~\ref{def:eps_representative}) for our sketching matrices $\{\Pi_j \in \R^{m \times k}\}_{j=1}^L$. This property implies that for any given vector, most of the matrices approximately preserves its length. In particular, we will consider the set of unit vectors $\{v_i = \frac{U(q-x_i))}{\|U(q-x_i)\|_2}\}_{i=1}^n$ for query $q \in \R^d$ and data points $\{x_1, x_2, \dots, x_n\}\subset \R^d$ i.e. for any point $x_i$, we have that most of the projections $\Pi_j$ satisfy 
\begin{align*}
& ~ (1-\epsilon)\|q-x_i\|_{A} \\
\leq & ~ \|\Pi_j(U(q-x_i))\|_2 \\
\leq & ~ (1+\epsilon)\|q-x_i\|_{A}.
\end{align*}
%\bigskip
%\\
For query $q \in \R^d, i \in [n]$, we will show that $\tilde{d}_{i}$ is a good estimate of $\|q-x_{i}\|_{A}$ with high probability. From the definition of $\tilde{d}_{i}$, we know that the $\tilde{d}_i = 0 = \|q-x_i\|_{A}$ is true when $q=x_{i}$. Therefore, we only need to consider the case when $q \neq x_{i}$. Let $v=\frac{U(q-x_{i})}{\|q-x_{i}\|_{A}}$.
\\From Lemma~\ref{lem:epsilon_representative}, we have that $\{\Pi_{j}\}_{j=1}^{L}$ is $(\epsilon,0.1)$-representative. So the set $\mathcal{J}$ defined as:
\begin{align*}
    \mathcal{J} \coloneqq \{j:  (1 - \epsilon) \leq \|\Pi_{j} U v\| \leq (1 + \epsilon)\}
\end{align*}
has size at least $0.9L$. We now define the random variables 
\begin{align*}
\tilde{y}_{i, r}:=\|\Pi_{j_{r}}U v\|
%\end{align*}
\text{~~~and~~~} 
%\begin{align*}
\tilde{z}_{i}:=\Median_{r \in [R]}\{\tilde{y}_{i, r}\}
\end{align*}
with $R,\{j_{r}\}_{r=1}^{R}$ defined in \textsc{Query} in Algorithm~\ref{alg:query}. We know  that $\tilde{d}_{i}=\|q-x_{i}\|_{A} \tilde{z}_{i}$ from the definition of $\tilde{d}_{i}$. Therefore, it is necessary and sufficient to bound the probability that $\tilde{z}_{i} \in[1-\epsilon, 1+\epsilon]$. To do this, let $W_{r}={\bf 1}\{j_{r} \in \mathcal{J}\}$ and $W=\sum_{r=1}^{R} W_{r}$. Furthermore, we have $\E[W] \geq 0.9 r$ and since $W_{r} \in\{0,1\}$, we have by Hoeffding's Inequality (Lemma~\ref{thm:blm13_hoeffding_inequality}).

\begin{align*}
    \Pr[W \leq 0.6 R] \leq \exp (-\frac{2(0.3 R)^{2}}{R}) \leq \frac{\delta}{n}
\end{align*}
from our definition of $r$. Furthermore, for all $k$ such that $j_{k} \in \mathcal{J}$, we have:
\begin{align*}
    1-\epsilon \leq \tilde{y}_{i, k} \leq 1+\epsilon .
\end{align*}
Therefore, in the event that $W \geq 0.6 R$, we have $(1-\epsilon) \leq \tilde{z}_{i} \leq(1+\epsilon)$. Hence, we get:
% \begin{align*}
%   \Pr[\tilde{d}_{i} \approx (1\pm \epsilon)\|q-x_{i}\|_{U}] \geq 1-\frac{\delta}{n} . 
% \end{align*}
\begin{align*}
   (1-\epsilon)\|q-x_{i}\|_{A} \leq \tilde{d}_{i} \leq(1+\epsilon)\|q-x_{i}\|_{A}. 
\end{align*}
with probability at least $1 -\frac{\delta}{n}$.

Taking a union bound over all $i \in [n]$, we get:
\begin{align*}
   \forall i \in[n], (1-\epsilon)\|q-x_{i}\|_{A} \leq \tilde{d}_{i} \leq(1+\epsilon)\|q-x_{i}\|_{A}. 
\end{align*}
with probability at least $1 -\delta$. 

This concludes the proof of correctness for the output of \textsc{QueryAll}.
\end{proof}

%% file: experiment.tex
% \vspace{-4mm}
\section{Evaluation}\label{sec:eval}
% \vspace{-6mm}
\begin{figure}[!h]
\centering
\subfloat[]{\includegraphics[width = 0.242\textwidth]{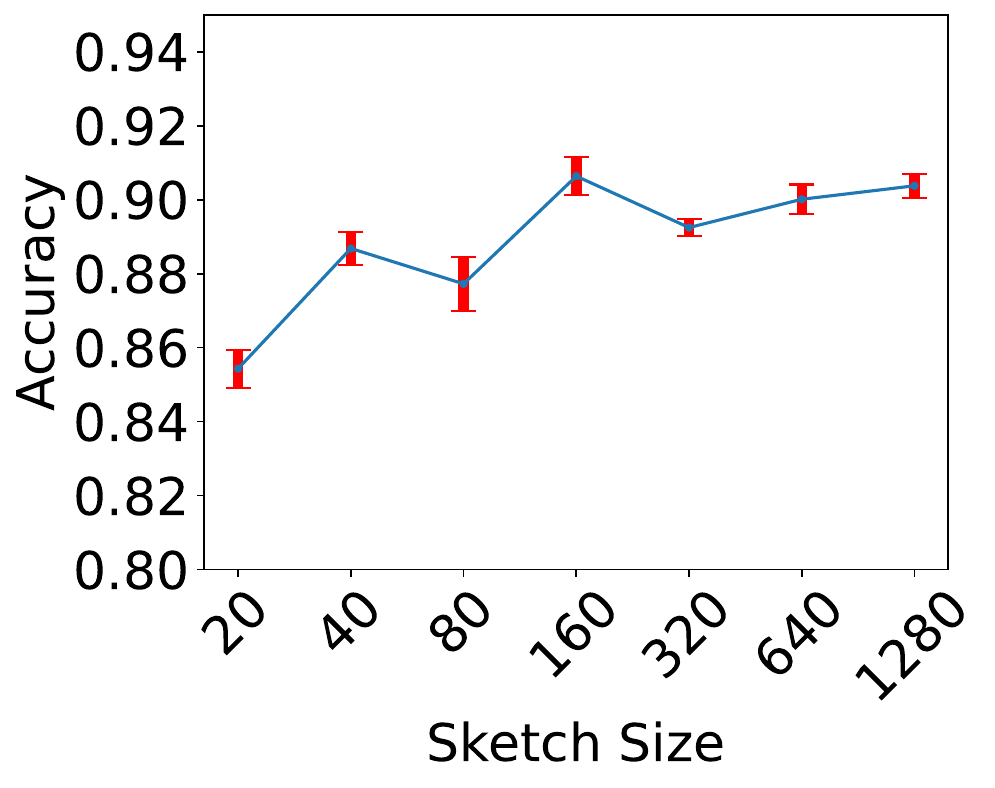}}
 % \hspace{2mm}
\subfloat[]{\includegraphics[width = 0.23\textwidth]{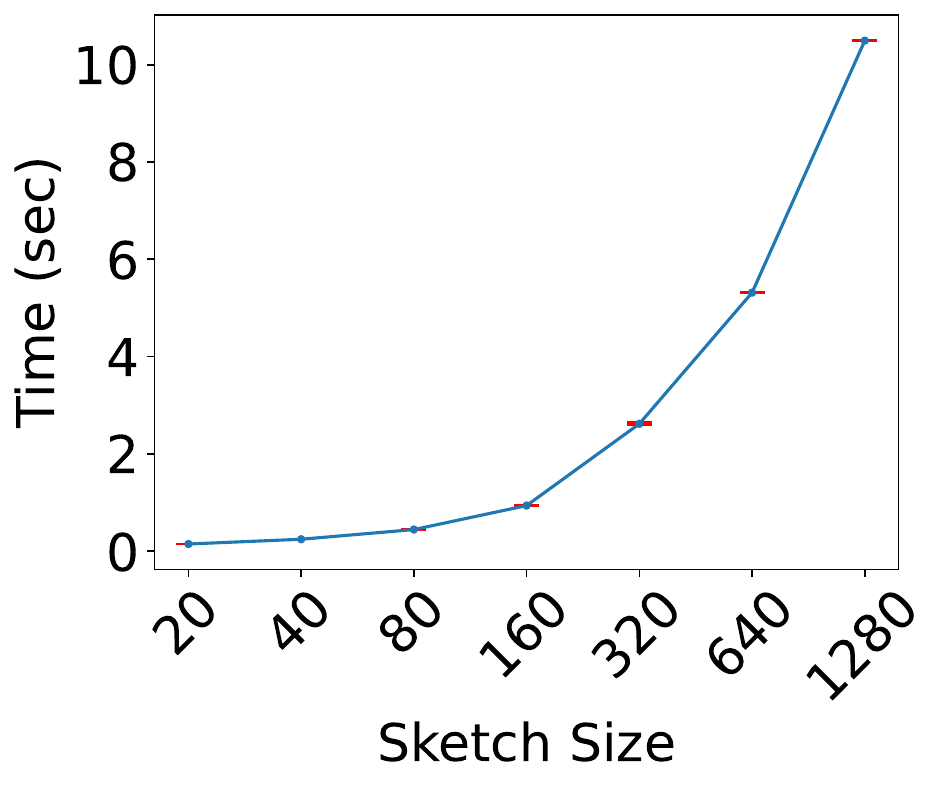}}
\caption{{ Benchmark results for (a) \textsc{QueryAll} accuracy under different $m$. (b) \textsc{QueryAll} time under different $m$.
}}
% \vspace{-2mm}
\label{fig:sketch_size}
\end{figure}
% \vspace{-4mm}

\begin{figure}[!h]
\centering
\subfloat[]{\includegraphics[width = 0.24\textwidth]{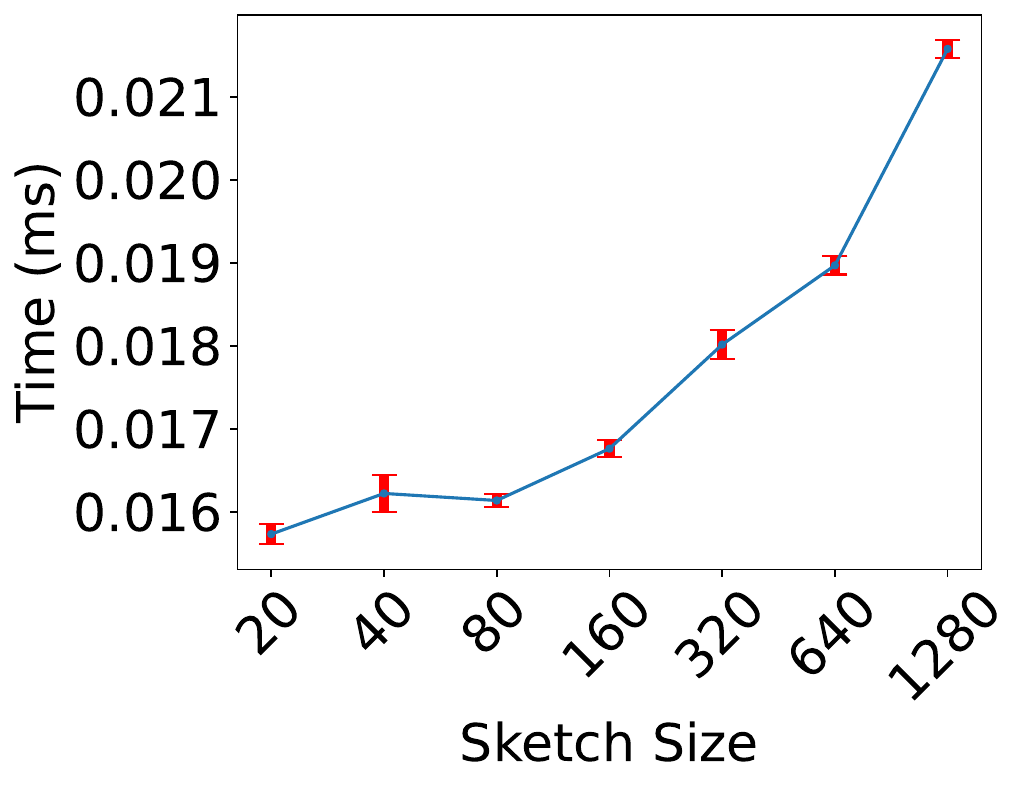}}
% \hspace{2mm}
\subfloat[]{\includegraphics[width = 0.235\textwidth]{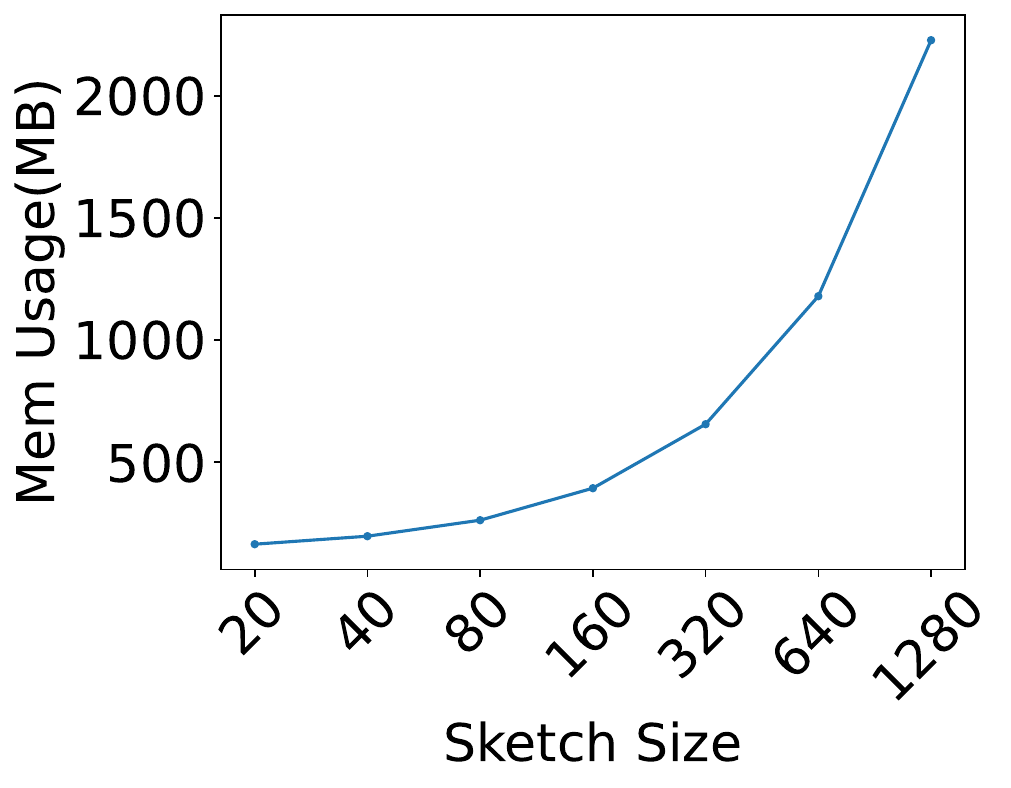}}
\caption{{ Benchmark results for 
(a) \textsc{QueryPair} time under different $m$. (b) Data structure memory usage under different $m$.}}
% \vspace{-4mm}
\label{fig:sketch_size2}
\end{figure}
%\vspace{-4mm}

\paragraph{Experiment setup.} We evaluate our algorithm with gene expression cancer RNA-Seq Data Set from UCI machine learning repository~\cite{an07} where
 $n=800$ and we use the first $d=5120$ columns. We run our experiments on a machine with Intel i7-13700K, 64GB memory, Python 3.6.9 and Numpy 1.22.2. We set the number of sketches $L = 10$ and sample $R = 5$ sketches during \textsc{QueryAll}.
We run \textsc{QueryAll} for $10$ times to obtain the average time consumption and the accuracy of \textsc{QueryAll}. We run 
\textsc{QueryPair} for $10000$ times to obtain the average time. The red error bars in the figures represent the standard deviation. We want to answer the following questions:
\begin{itemize}
    \item How is the execution time affected by the sketch size $m$?
    \item How is the query accuracy affected by the sketch size $m$?
    \item How is the memory consumption of our data structure affected by the sketch size $m$?
\end{itemize}

\paragraph{Query Accuracy.} From the part (a) in Figure~\ref{fig:sketch_size}, we know that the \textsc{QueryAll} accuracy increases from $83.4\%$ to $90.4\%$ as the sketch size increases from $20$ to $1280$. And we can reach around $90\%$ distance estimation accuracy when the sketch size reaches $320$.

\paragraph{Execution time.} 
From the part (b) in Figure~\ref{fig:sketch_size}, \textsc{QueryAll} time grows from $0.15$ second to $10.49$ second as sketch size increases.
From the part (a) in Figure~\ref{fig:sketch_size2}, \textsc{QueryPair} time increases from $0.015$ millisecond to $\sim 0.02$ milliseconds as sketch size increases. 
The query time growth follows that larger sketch size leads to higher computation overhead on $\Pi_{j} U q$.
Compared with the baseline whose sketch size $m=1280$, when sketch size is $320$, the \textsc{QueryAll} is $8.5 \times$ faster and our data structure still achieves $\sim 90\%$ estimation accuracy in \textsc{QueryAll}. 

\paragraph{Memory Consumption.} From the part (b) in Figure~\ref{fig:sketch_size2}, we find that the memory consumption
increases from $164$MB to $2228$MB as the sketch size increases from $20$ to $1280$. Larger sketch size means that more space is required to store the precomputed $\Pi_j U$ and $\Pi_j U x_i$ matrice. Compared with the baseline sketch size  $m=1280$, when sketch size is $320$, the memory consumption is $3.06 \times$ smaller. 

%% file: concl.tex
\section{Conclusion}\label{sec:conclusion}
 
Mahalanobis metrics have become increasingly used in machine learning algorithms like nearest neighbor search and clustering. Surprisingly, to date, there has not been any work in applying sketching techniques on algorithms that use Mahalanobis metrics. We initiate this direction of study by looking at one important application, Approximate Distance Estimation (ADE) for Mahalanobis metrics. We use sketching techniques to design a data structure which can handle sequences of adaptive and adversarial queries. Furthermore, our data structure can also handle an online version of the ADE problem, where the underlying Mahalanobis distance changes over time. Our results are the first step towards using sketching techniques for Mahalanobis metrics. We leave the study of using our data structure in conjunction with online Mahalanobis metric learning algorithms like \cite{ssn04} to future work. From our perspective, given that our work is theoretical, it doesn't lead to any negative societal impacts.

\section{Acknowledgements}
Aravind Reddy was supported in-part by NSF grants CCF-1652491, CCF-1934931, and CCF-1955351 during the preparation of this manuscript.

%% file: app_tools.tex
% \paragraph{Roadmap.} We introduce some probability tools in Appendix~\ref{sec:tools}. 
% % We first present the missing proofs for Section~\ref{sec:sketch} in Appendix~\ref{sec:sketch_app}. 
% We then present the proofs for online adaptive Mahalanobis distance maintenance in Appendix~\ref{sec:proof_running_time}. We present the supplementary algorithm description for \textsc{SampleExact} in Appendix~\ref{sec:approx_sample}.
% % \lianke{do not say proof for section??? use section for section and appendix. change the subsection name too}

%\lianke{below is for aaai2023}
% \paragraph{Roadmap.} We first introduce some probability tools in Section~\ref{sec:tools}. 
% We first present the missing proofs for Section~\ref{sec:sketch} in Appendix~\ref{sec:sketch_app}. 
We then present the proofs for online adaptive Mahalanobis distance maintenance. We present the supplementary algorithm description for \textsc{SampleExact} in Section~\ref{sec:proof_running_time}.
We propose sampling algorithm based on our Mahalanobis distance estimation data structure in Section~\ref{sec:approx_sample}.
Then we present supplementary experiments in Section~\ref{sec:exp_app}.
% \lianke{do not say proof for section??? use section for section and appendix. change the subsection name too}

% \vspace{-4mm}

\iffalse

We will make use of Hoeffding’s Inequality:
\begin{lemma}[\cite{h63}]\label{thm:blm13_hoeffding_inequality}
   Let $X_{1}, \ldots, X_{n}$ be independent random variables such that $X_{i} \in\left[a_{i}, b_{i}\right]$ almost surely for $i \in[n]$ and let $S=\sum_{i=1}^{n} X_{i}-\E[X_{i}]$. Then, for every $t>0$ :
\begin{align*}
   \Pr[S \geq t] \leq \exp \left(-\frac{2 t^{2}}{\sum_{i=1}^{n}\left(b_{i}-a_{i}\right)^{2}}\right) .
\end{align*}
\end{lemma}
\fi

\begin{algorithm}[!ht]\caption{Segment Tree}\label{alg:segment_tree}
\begin{algorithmic}[1]
\State {\bf data structure} {\textsc{TreeNode}}
\State {\bf members}
\State \hspace{4mm} $\textsc{TreeNode}$ $l, r$ \Comment{Left and right child node}
\State \hspace{4mm} $v \in \R$ \Comment{Value stored in current node}
\State \hspace{4mm} $L, R \in \mathbb{N}_{+} $ \Comment{The left and right boundary of node interval }
\State {\bf end members}
\State {\bf end data structure} 
\State 
\State {\bf data structure} {\textsc{Tree}} 
\State {\bf members}
\State \hspace{4mm} $t \in \textsc{TreeNode}$ \Comment{$t$ is the root node }
\State {\bf end members}
\State
\Procedure{Init}{$a \in \R^{n}, c \in \textsc{TreeNode}, L \in \mathbb{N}_{+}, R \in \mathbb{N}_{+}$}.
    \State $c.L \gets L, c.R \gets R$ \Comment{Assign the interval of current node}
    \If{$L = R$}
        \State $c.v \gets a_l$  \Comment{Assign current node value}
    \Else 
        \State $m \gets (l + r)/2$
        \State $\textsc{Init}(a, c.l, l, m)$
        \State $\textsc{Init}(a, c.r, m+1, r)$
        \State $c.v \gets c.l.v + c.r.v$ \Comment{Assign current node value}
    \EndIf
\EndProcedure
\State
\Procedure{Query}{$c \in \textsc{TreeNode}, L \in \mathbb{N}_{+}, R \in \mathbb{N}_{+} $}
    \If{$ c.R < L$ or $c.L > R$} \Comment{if the requested range is outside interval of current node}
        \State \Return NULL
    \EndIf
    \If{$c.L \leq L $ and $c.R \geq R$} \Comment{if current node interval is completely inside requested range}
        \State \Return $c.v$
    \EndIf
    \State \Return $\textsc{Query}(c.l, L, R) + \textsc{Query}(c.r, L, R)$ 
\EndProcedure
\State
\Procedure{Update}{$c \in \textsc{TreeNode}, i \in [n], v' \in \R $}
    \If{$c.L = c.R$}\Comment{If current node is leaf node}
        \State $c.v \gets v'$
    \EndIf
    \If{$L \leq i$ and $(c.L + c.R)/2 \geq i$}
        \State $\textsc{Update}(c.l, i, v')$ \Comment{Update left child node}
    \Else 
         \State $\textsc{Update}(c.r, i, v')$ \Comment{Update right child node}
    \EndIf
    \State $c.v \gets c.l.v + c.r.v$ \Comment{Update current node}
\EndProcedure
\State {\bf end data structure}
\end{algorithmic}
\end{algorithm}

\begin{algorithm}[!ht]\caption{Online Adaptive Mahalanobis Distance Maintenance: SampleExact. }\label{alg:sample_exact}
\begin{algorithmic}[1]
% \State \lianke{This algorithm block is new}
\State {\bf data structure} {\sc MetricMaintenance} \Comment{Theorem~\ref{thm:main}}
% \State \lianke{If we want to compute the exact $\sum_{\ell = i}^{j} \| q - x_{\ell}\|_{A}$, we should not use sketch.}
\Procedure{SubSum}{$q \in \R^{d}, i \in [n], j \in [n]$} \Comment{Compute $\sum_{\ell = i}^{j} \| q - x_{\ell}\|_{A}$. Lemma~\ref{lem:subsum}}.
    \State $s \gets \textsc{tree}.\textsc{Query}(i, j)$  \Comment{$s \in \R^d$}
    \State $r \gets (j-i+1) \cdot \| U q \|_2 + 2 q^{\top} U^{\top} U s  + \| U  s\|_2$
    \State \Return $r$
\EndProcedure
% \Procedure{SubSumApprox}{$q \in \R^{d}, i \in [n], j \in [n]$} \Comment{ Lemma~\ref{lem:subsum_approx}}.
%     \State $s \gets \textsc{tree}.\textsc{Query}(i, j)$  \Comment{$s \in \R^d$}
%     \State $r \gets (j-i+1) \cdot \| \Pi U q \|_2 + 2 q^{\top} U^{\top} U s  + \| \Pi U  s\|_2$ \lianke{do we need the median of $R$ sketching here?}
%     \State \Return $r$
% \EndProcedure
\Procedure{SubSampleExact}{$q \in \R^d, l \in \mathbb{N}_{+}, r \in \mathbb{N}_{+}$}
    \While{$l < r$}
        \State $m = l + (r - l)/2$
        \State Sample a random value $v$ from $[0,1]$ uniformly.
        \State $t = \frac{\textsc{SubSum}(q, l, m)}{\textsc{SubSum}(q, l, r)}$\Comment{Threshold to decide next round sampling}
        \If{$v < t$}
            \State \Return $\textsc{SubSampleExact}(q, l, m)$
        \Else
            \State \Return $\textsc{SubSampleExact}(q, m+1, r)$
        \EndIf 
    \EndWhile
    \State \Return $l$
\EndProcedure
\Procedure{SampleExact}{$q \in \R^d$}\Comment{Lemma~\ref{lem:sample_exact}}
    \State  \Return $\textsc{SubSampleExact}(q, 1, n)$
\EndProcedure
% \Procedure{SubSampleApprox}{$q \in \R^d, l \in \mathbb{N}_{+}, r \in \mathbb{N}_{+}$}
%     \While{$l < r$}
%         \State $m = l + (r - l)/2$
%         \State Sample a random value $v$ from $[0,1]$ uniformly.
%         \State $t = \frac{\textsc{SubSumApprox}(q, l, m)}{\textsc{SubSumApprox}(q, l, r)}$\Comment{Threshold to decide next round sampling}
%         \If{$v < t$}
%             \State \Return $\textsc{SubSampleApprox}(q, l, m)$
%         \Else
%             \State \Return $\textsc{SubSampleApprox}(q, m+1, r)$
%         \EndIf 
%     \EndWhile
%     \State \Return $l$
% \EndProcedure
% \Procedure{SampleApprox}{$q \in \R^d$}\Comment{Lemma~\ref{lem:sample_approx}}
%     \State \Return $\textsc{SubSampleApprox}(q, 1, n)$
% \EndProcedure
\State {\bf end data structure}
\end{algorithmic}
\end{algorithm}
% \Danyang{lianke, add a paragraph here to summarize the section.}

%% file: jl_sketch_app.tex
% \section{Algorithm for Section~\ref{sec:sketch}} \label{sec:sketch_app}

\begin{algorithm}[!ht]\caption{Formal version of Algorithm~\ref{alg:jl_guarantees}. JL sketch for approximate Mahalanobis distance estimation. We remark that the proof on page 5 mentions Line~\ref{alg:jl_sketch_init_matrix_mul} and Line~\ref{alg:jl_sketch_query_matrix_mul} which are in fact with respect to the Formal version (this Algorithm~\ref{alg:jl_guarantees_full}). Line~\ref{alg:jl_sketch_init_matrix_mul} in Algorithm~\ref{alg:jl_guarantees_full} is corresponding to Line~\ref{alg:jl_sketch_init_matrix_mul:informal} in Algorithm~\ref{alg:jl_guarantees} (This step takes $O((d+m)k)$ time). Line~\ref{alg:jl_sketch_query_matrix_mul} in Algorithm~\ref{alg:jl_guarantees_full} is corresponding to Line~\ref{alg:jl_sketch_query_matrix_mul:informal} in Algorithm~\ref{alg:jl_guarantees} (This step takes $O((d+m)k)$ time).   }
\label{alg:jl_guarantees_full}
\begin{algorithmic}[1]
\State {\bf data structure} {\sc JLMonCarMaintenance} \Comment{Theorem~\ref{thm:jl_guarantees}}
\State {\bf members}
\State \hspace{4mm} $d, k, m, \in \mathbb{N}_{+}$ \Comment{$d$ is the dimension of data, $m$ is the sketch size, $n$ is the number of points}
\State \hspace{4mm} $U \in \R^{k \times d}$
\State \hspace{4mm} $\{\tilde{x}_{i}\} \in \R^{m}$ for $i \in [n]$ \Comment{Sketch of the $n$ points}
\State {\bf end members}
\State
\Procedure{Initialize}{$U \in \R^{k \times d}, \{x_i\}_{i\in [n]} \subset \R^d, \epsilon \in (0,1), \delta \in (0,1)$} 
    \State $m \gets  \Omega(\frac{\log n }{\epsilon^2})$  \Comment{Initialize the sketch size}
    \State $U \gets U$
    \For{$i = 1 \to n$}
        \State $x_i \gets x_i$
    \EndFor
    \State Let $\Pi \in \R^{m \times k}$ with entries drawn iid from $\mathcal{N}(0,1/m)$. \Comment{AMS or CountSketch also valid}
    \For{$i \in [n]$}
        \State $\tilde{x}_{i} \gets \Pi Ux_i$ \label{alg:jl_sketch_init_matrix_mul} \Comment{This is corresponding to Line~\ref{alg:jl_sketch_init_matrix_mul:informal} in Algorithm~\ref{alg:jl_guarantees}}
    \EndFor
\EndProcedure
\State
\Procedure{Query}{$q \in \R^d$}%\label{alg:query_all_procedure}
    \For{$i \in [n]$}
        \State $ \tilde{d}_{i} \gets \|\Pi Uq-\tilde{x}_{i}\|_2$\label{alg:jl_sketch_query_matrix_mul} \Comment{This is corresponding to Line~\ref{alg:jl_sketch_query_matrix_mul:informal} in Algorithm~\ref{alg:jl_guarantees}}
    \EndFor
    \State \Return $\{\tilde{d}_i\}_{i=1}^n$
\EndProcedure
\State {\bf end data structure}
\end{algorithmic}
\end{algorithm}

%% file: time_app.tex
\section{Proofs for Online Adaptive Mahalanobis Distance Maintenance}\label{sec:proof_running_time}
 
In this section, we provide lemmas and proofs for the initialization, update, query, and query pair operations.
Algorithm \ref{alg:jl_guarantees_full} is the preliminary version of approximate Mahalanobis distance estimation based on JL sketch. The \textsc{Initialize} operation sample a sketch matrix $\Pi \in \R^{m \times k}$ and compute $\wt{x}_i = \Pi U x_i$ for all $i \in [n]$. The \textsc{Query} operation compute the estimated distance by $\wt{d}_i = \|\Pi U q - \wt{x}_i\|_2$ for all $i \in [n]$.

\begin{algorithm}[!ht]\caption{Formal version of Algorithm~\ref{alg:query}. Online Adaptive Mahalanobis Distance Maintenance: queries.}\label{alg:query_full}
\begin{algorithmic}[1]
\State {\bf data structure} {\sc MetricMaintenance} \Comment{Theorem~\ref{thm:main}}
% \Procedure{Query}{$q \in \R^d, i \in [n]$}\label{alg_query_procedure} \Comment{Lemma~\ref{lem:query}}
%     \State $R \gets O(\log(n/\delta))$ \Comment{Number of sampled sketches}
%     \State Sample $j_1, j_2, \dots, j_R$ i.i.d. with replacement from $[L]$.\label{alg:query_sample_j}
%     \For{$r = 1\to R$}\label{alg:query_norm_loop_start} %\Zhao{Change this to $r=1 \to R$ }
%         \State  $d_{i,r} \gets \|\Pi_{j_r}Uq- \tilde{x}_{i,r}\|_2$\label{alg:query_norm}
%     \EndFor\label{alg:query_norm_loop_end}
%     % \State $\tilde{d}_i \gets \Median(\{d_{i,r}\}_{r=1}^R)$\label{alg:query_median} 
%     \State $\tilde{d}_i \gets \Median_{r\in [R]} \{ d_{i,r} \} $\label{alg:query_median} 
%     %\lianke{should be $\Median(\{y_{i,k}\}_{k=1}^r)$?}\Aravind{Yes, that's right.}
%     \State \Return $\tilde{d}_i$
% \EndProcedure
% \State
\Procedure{QueryPair}{$i,j \in [n]$}\label{alg:query_pair}\Comment{Lemma~\ref{lem:query_pair}}
    \State $R \gets O(\log(n/\delta))$ \Comment{Number of sampled sketches}
    \State {\bf for }{$r = 1 \to R$}\label{alg:query_pair_for_loop} 
        \State \hspace{8mm} $p_r \gets \| \tilde{x}_{i,r} - \tilde{x}_{j,r} \|_2$ 
    \State{\bf end for}
    \State $p \gets \Median_{r \in [R]}\{ p_r\}$ \label{alg:query_pair_median}
    \State \Return $p$
\EndProcedure
\State
\Procedure{QueryAll}{$q \in \R^d$}\label{alg:query_all_procedure} \Comment{Lemma~\ref{lem:query_all}}
    \State $R \gets O(\log(n/\delta))$  \Comment{Number of sampled sketches}%\Zhao{Change small $r$ to Big $R$}
    \State Sample $j_1, j_2, \dots, j_R$ i.i.d. with replacement from $[L]$.\label{alg:query_all_sample_j} %\Zhao{Change $r$ to $R$, change $l$ to $L$}
    \For{$i = 1 \to n$}
        \For{$r = 1 \to R$}
            \State $ d_{i,r} \gets \|\Pi_{j_r}Uq-\tilde{x}_{i,r}\|_2$\label{alg:query_all_norm} %\Zhao{Change $k \in [r]$ to $r \in [R]$}
        \EndFor
    \EndFor
    \For{$i = 1 \to n$}
        \State $\tilde{d}_i \gets \Median_{r\in [R]} \{ d_{i,r} \}$\label{alg:query_all_median}%\lianke{should be $\Median(\{y_{i,k}\}_{k=1}^r)$?}%Aravind{Yes}
    \EndFor
    \State \Return $\{\tilde{d}_i\}_{i=1}^n$
\EndProcedure
\State {\bf end data structure}
\end{algorithmic}
\end{algorithm}

In Algorithm~\ref{alg:mahalanobis_maintenance_full}, \textsc{Initialize} samples $L = O((d+ \log \frac{1}{\delta}) \log(d/\epsilon))$ sketching matrice and compute the $\wt{x}_{i,j} = \Pi_j U x_i$ for all $n$ data points and $L$ sketching matrice. \textsc{UpdateU} uses the sparse update matrix $B$ to update the distance matrix $U$ and  $\wt{x}_{i,j}$. \textsc{UpdateX} updates the corresponding $\wt{x}_{i, j}$ with the new data point $z$ and target index $i$ for all sketching matrix $j \in [L]$.

In Algorithm~\ref{alg:query_full}, \textsc{QueryPair} samples $R = O(\log(n/\delta))$ sketching matrix indexes and uses them to compute a list of estimated distance between $x_i$ and $x_j$ and output the median in the end. \textsc{QueryAll} samples $R = O(\log(n/\delta))$ sketching matrix indexes and compute $ d_{i,r} = \|\Pi_{j_r}Uq-\tilde{x}_{i,r}\|_2$ for all data point indexes $i \in [n]$ and sketching matrix indexes $r \in [R]$. Then \textsc{QueryAll} outputs the median value $\tilde{d}_i \gets \Median_{r\in [R]} \{ d_{i,r} \}$ as the estimated distance between $q$ and $x_i$ for all data point indexes $i \in [n]$.
% \lianke{add more alg description line by line explanation turn math into english.}
%\Danyang{lianke TODO}

% In Section~\ref{sec:proof_time_init}, we present the proof for the time complexity of \textsc{Initialize}.
% In Section~\ref{sec:proof_time_update}, we present the proof for the time complexity of \textsc{UpdateU}.
% In Section~\ref{sec:proof_time_updatex}, we present the proof for the time complexity of \textsc{UpdateX}
% In Section~\ref{sec:proof_lemma_query_pair}, we present the proof for \textsc{QueryPair}.
% In Section~\ref{app:proof_lemma_query_all}, we present the proof for \textsc{QueryAll}.
% In Section~\ref{sec:proof_subsum}, we present the proof for \textsc{SubSum}.
% In Section~\ref{sec:proof_sample_exact}, we present the proof for \textsc{SampleExact}.

\begin{algorithm}[h]\caption{Formal version of Algorithm~\ref{alg:mahalanobis_maintenance}. Mahalanobis Pseudo-Metric Maintenance: members, initialize and update}
\label{alg:mahalanobis_maintenance_full}
\begin{algorithmic}[1]
\State {\bf data structure} {\sc MetricMaintenance} \Comment{Theorem~\ref{thm:main}}
\State {\bf members}
\State \hspace{4mm} $L, m, k$\Comment{$L$ is the number of sketches, $m$ is the sketch size }
\State \hspace{4mm} $ d, n\in \mathbb{N}_{+}$ \Comment{ $n$ is the number of points, $d$ is dimension} 
\State \hspace{4mm} $U \in \R^{k \times d}$
\State \hspace{4mm} $x_i \in \R^d$ for $i \in [n]$
\State \hspace{4mm} $\{\tilde{x}_{i,j}\} \in \R^{m}$ for $i \in [n], j \in [L]$ \Comment{The sketch of data points}
\State \hspace{4mm} $\textsc{Tree}$  $\textsc{tree}$ \Comment{Segment tree(Alg~\ref{alg:segment_tree}) $\textsc{tree}.\textsc{Query}(i, j)$ to obtain $\sum_{\ell = i}^{j} x_{\ell}$}
\State {\bf end members}
\State
\Procedure{Initialize}{$U \in \R^{k \times d}, \{x_i\}_{i \in [n]} \subset \R^d, \epsilon \in (0,1), \delta \in (0,1)$} \Comment{Lemma~\ref{lem:init}}
    \State $m \gets O(\frac{1}{\epsilon^2})$  \Comment{Initialize the sketch size}
    \State $L \gets O((d+ \log \frac{1}{\delta}) \log(d/\epsilon))$ \Comment{Initialize the number of copies of sketches}
    \State $U \gets U$
    \For{$i = 1\to n$} \label{alg:init_x_assign_start}
        \State  $x_i \gets x_i$ 
    \EndFor \label{alg:init_x_assign_end}
    \State For $j \in [L]$, let $\Pi_j \in \R^{m \times k}$ with entries drawn iid from $\mathcal{N}(0,1/m)$.\label{alg:init_pi_drawn} %\Zhao{In this intialize function, the $l$ appears in 3 places, please change all of them to $L$}
    \For{$i = 1\to n$}\label{alg:init_assign_tilde_x_for_start}
        \For{$j = 1\to L$}
            \State  $\tilde{x}_{i,j} \gets \Pi_jUx_i$\label{alg:init_assign_tilde_x}
        \EndFor
    \EndFor \label{alg:init_assign_tilde_x_for_end} 
    \State $\textsc{tree}.\textsc{Init}(\{x_i\}_{i \in [n]}, \textsc{tree}.t, 1, n)$ \Comment{Initialize the segment tree using all data points.}
    % \State \lianke{add segment tree init for $\{x_i\}_{i=1}^{n}$}
\EndProcedure
\State
\Procedure{UpdateU}{$u \in \R^d, a \in [k]$} \Comment{We can consider sparse update for $U$, Lemma~\ref{lem:update}}.
    \State $B \gets \{0\}_{k \times d}$
    \State $B_{a} \gets u^\top$\label{alg:update_b_a} \Comment{$B_a$ denotes the $a$'th row of $B$.}
    \State $U \gets U + B$ \label{alg:update_U_add_B}
    \For{$i = 1\to n$}\label{alg:update_x_i_j_loop_start}
        \For{$j = 1\to L$}
            \State  $\tilde{x}_{i,j} \gets \tilde{x}_{i,j} + \Pi_jBx_i$ \label{alg:update_x_i_j}
        \EndFor
    \EndFor\label{alg:update_x_i_j_loop_end}
\EndProcedure
\State
\Procedure{UpdateX}{$z \in \R^d, i \in [n]$} \Comment{ Lemma~\ref{lem:updateX}}.
    \State $x_i \gets z$
    \For{$j = 1\to L$}\label{alg:updateX_start_loop}
        \State  $\tilde{x}_{i,j} \gets \Pi_j U z$ 
    \EndFor\label{alg:updatex_end_loop}
    \State $\textsc{tree}.\textsc{Update}(\textsc{tree}.t, i, z)$
\EndProcedure

\State {\bf end data structure}
\end{algorithmic}
\end{algorithm}

%% file: sample_app.tex
\section{Sampling}\label{sec:approx_sample}

In this section, we provide lemmas to support the sampling from the $n$ data points based on their distances to the query point $q$. 
We defer the data structure description to Algorithm~\ref{alg:segment_tree} and Algorithm~\ref{alg:sample_exact} in the Appendix due to space limitation. It allows querying which of the stored segments contain a given point.
The goal of Algorithm~\ref{alg:sample_exact} is to implement \textsc{SampleExact} functionality.
\textsc{SubSum} operation can leverage the segment tree data structure to obtain the $\sum_{\ell = i}^{j} x_l$ in only $O(\log(n))$ time and compute the  $\sum_{l=i}^{j} \| q - x_l \|_A$ in $O(\log n + kd)$ time.
\textsc{SampleExact} executes like a binary-search fashion and calls \textsc{SubSum} operation for $O(\log(n))$ times to determine the final data point index.

\begin{lemma}[SubSum]\label{lem:subsum}
Given a query point $q \in \R^{d}$, two indexes $i \in [n], j \in [n]$ and $i < j$, \textsc{SubSum} output $\sum_{l=i}^{j} \| q - x_l \|_A$ in $O(\log n + kd)$ time.
\end{lemma}

\begin{proof}
{\bf Proof of Correctness.} We have:
\begin{align*}
    & ~ \sum_{l=i}^{j} \| q - x_l \|_A \\
    = & ~ \sum_{\ell = i}^{j} q^{\top} A q - 2 q^{\top} A \sum_{\ell = i}^{j} x_{\ell} + \sum_{\ell = i}^{j} x_{\ell}^{\top} A x_{\ell} \\
    = & ~ (j - i + 1) \cdot \| U q\|_2 - 2 q^{\top} U^{\top} U \cdot \textsc{tree}.\textsc{Query}(i, j) +  \\ 
    & ~ \| U \cdot \textsc{tree}.\textsc{Query}(i, j) \|_2 \\
    = &~ (j-i+1) \cdot \| U q \|_2 + 2 q^{\top} U^{\top} U s  + \| U  s\|_2
\end{align*}
Therefore, the \textsc{SubSum} operation can correctly output $\sum_{l=i}^{j} \| q - x_l \|_A$.
$\sum_{l=i}^{j} \| q - x_l \|_A$.

\vspace{2mm}
{\bf Proof of Running Time.} 
The \textsc{SubSum} has four steps:
\begin{itemize}
    \item $s \gets \textsc{tree}.\textsc{Query}(i,j)$ takes $O(\log n)$ time to return $\sum_{\ell = i}^{j} x_{\ell}$ as a segment tree.
    \item  $\| U q\|_2$ takes $O(kd)$ to compute the $U \in \R^{k \times d}$ and $q \in \R^{d}$ multiplication.
    \item  $q^{\top} U^{\top} U s$ takes $O(kd)$ to compute $U \cdot s$, $O(kd)$ time to compute $U^{\top} \cdot (Us)$, and $O(kd)$ time to compute $q^{\top} \cdot (U^{\top} U s)$.
    \item  $\| U s \|_2 $ takes $O(kd)$ time to compute $U \in \R^{k \times d}$ and $s \in \R^{d}$ multiplication.
\end{itemize}
Therefore, the overall time complexity is:
\begin{align*}
    &~O(\log n) + O(kd) + O(kd) + O(kd) + O(kd)\\ 
    =&~ O(\log n + kd)
\end{align*}
\end{proof}

\iffalse
\lianke{leave the approx sampling for next round.}
\begin{lemma}[SubSumApprox]\label{lem:subsum_approx}
Given a query point $q \in \R^{d}$, two indexes $i \in [n], j \in [n]$ and $i < j$, \textsc{SubSumApprox} output $b$ such that 
\begin{align*}
  (1-\epsilon) \sum_{l=i}^{j} \| q - x_l \|_A \leq   b \leq (1+\epsilon) \sum_{l=i}^{j} \| q - x_l \|_A  
\end{align*}
in $O(???)$ time.
\end{lemma}
\begin{proof}
{\bf Proof of Correctness.} We have:
\begin{align*}
    & ~ \sum_{l=i}^{j} \| q - x_l \|_A \\
    = & ~ \sum_{\ell = i}^{j} q^{\top} A q - 2 q^{\top} A \sum_{\ell = i}^{j} x_{\ell} + \sum_{\ell = i}^{j} x_{\ell}^{\top} A x_{\ell} \\
    = & ~ (j - i + 1) \cdot \| U q\|_2 - 2 q^{\top} U^{\top} U \cdot \textsc{tree}.\textsc{Query}(i, j) + \| U \cdot \textsc{tree}.\textsc{Query}(i, j) \|_2 \\
    = &~ r \gets (j-i+1) \cdot \| U q \|_2 + 2 q^{\top} U^{\top} U s  + \| U  s\|_2
\end{align*}
\end{proof}
\fi
 
We present how to sample a data point based on the distance between a query point $q \in \R^d$ and the data points in the data structure.

\begin{lemma}[SampleExact]\label{lem:sample_exact}
Given a query point $q \in \R^{d}$, \textsc{SampleApprox} samples an index $i \in [n]$ with probability $d_i / \sum_{j=1}^n d_j$ in $O(\log^2 n + kd \log n)$ time.
\end{lemma}
 
\begin{proof}
{\bf Proof of Correctness.}
In Algorithm~\ref{alg:sample_exact}, we know that \textsc{SampleExact} calls \textsc{SubSampleExact} between range $[1, n]$. At each iteration, \textsc{SubSampleExact} filters out either left or right half of the current range based on the sum of Mahalanobis distance of left half and right half. Therefore, after $T = O(\log n)$ iterations, let $[l_{j}, r_{j}]$ denote the sample range at iteration $j$, and we know the final index $i$ is sampled with the probability :
\begin{align*}
    &~\frac{\sum_{j = l_{1}}^{r_{1}} d_{j}}{\sum_{j = 1}^{n} d_{j}} \cdot \frac{\sum_{j = l_{2}}^{r_{2}} d_{j}}{\sum_{j = l_{1}}^{r_{1 }} d_{j}} \cdot \cdots \cdot \frac{ d_{i}}{\sum_{j = l_{\ell_{T-1} }}^{r_{\ell_{T-1} }} d_{j}} \\
    =&~ d_i / \sum_{j=1}^n d_j
\end{align*}
{\bf Proof of Running Time.} 
The time complexity of \textsc{SampleExact} is dominated by calling \textsc{SubSum} for $T = O(\log n)$ iterations. From Lemma~\ref{lem:subsum}, we know each \textsc{SubSum} operation takes $O(\log n + kd)$ time. Therefore, we have the time complexity of \textsc{SampleExact} is $O(\log^2 n + kd \log n)$

This completes the proof.
\end{proof}

\iffalse
\begin{lemma}[SampleApprox]\label{lem:sample_approx}
Given a query point $q \in \R^{d}$, \textsc{SampleApprox} samples an index $i \in [n]$ with probability $\wt{d}_i / \sum_{j=1}^n \wt{d}_j$.
\end{lemma}
\begin{proof}
We have
\end{proof}
\fi

%% file: exp_app.tex
\section{More Experiments}\label{sec:exp_app}

%\section{Evaluation}
\begin{figure}[!h]
\centering
\subfloat[]{\includegraphics[width = 0.25\textwidth]{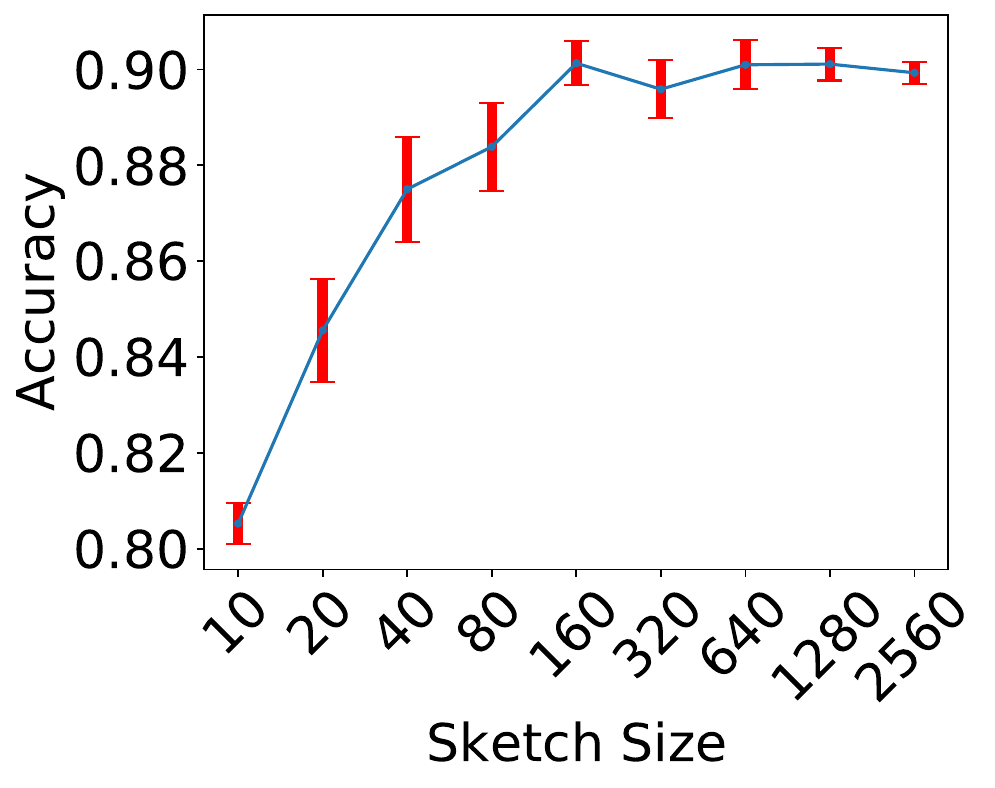}}
\hspace{2mm}
\subfloat[]{\includegraphics[width = 0.25\textwidth]{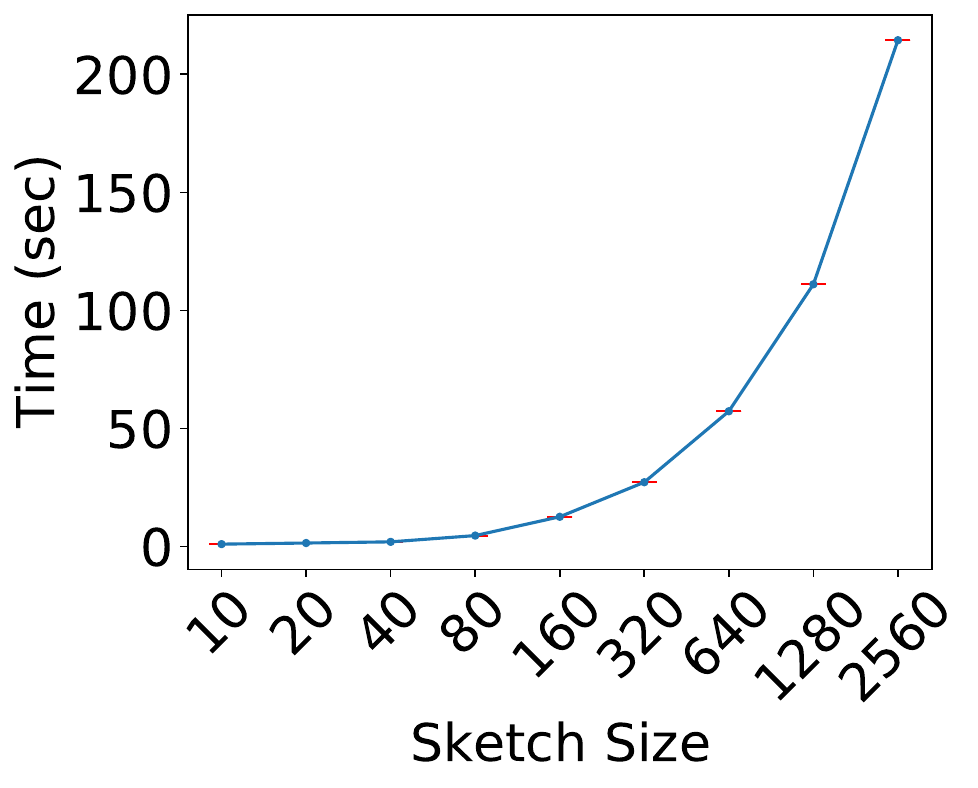}}
\caption{{ Benchmark results for (a) \textsc{QueryAll} accuracy under different $m$. (b) \textsc{QueryAll} time under different $m$.
}}
% \vspace{-4mm}
\label{fig:sketch_size_app}
\end{figure}

\begin{figure}[!h]
\centering
\subfloat[]{\includegraphics[width = 0.25\textwidth]{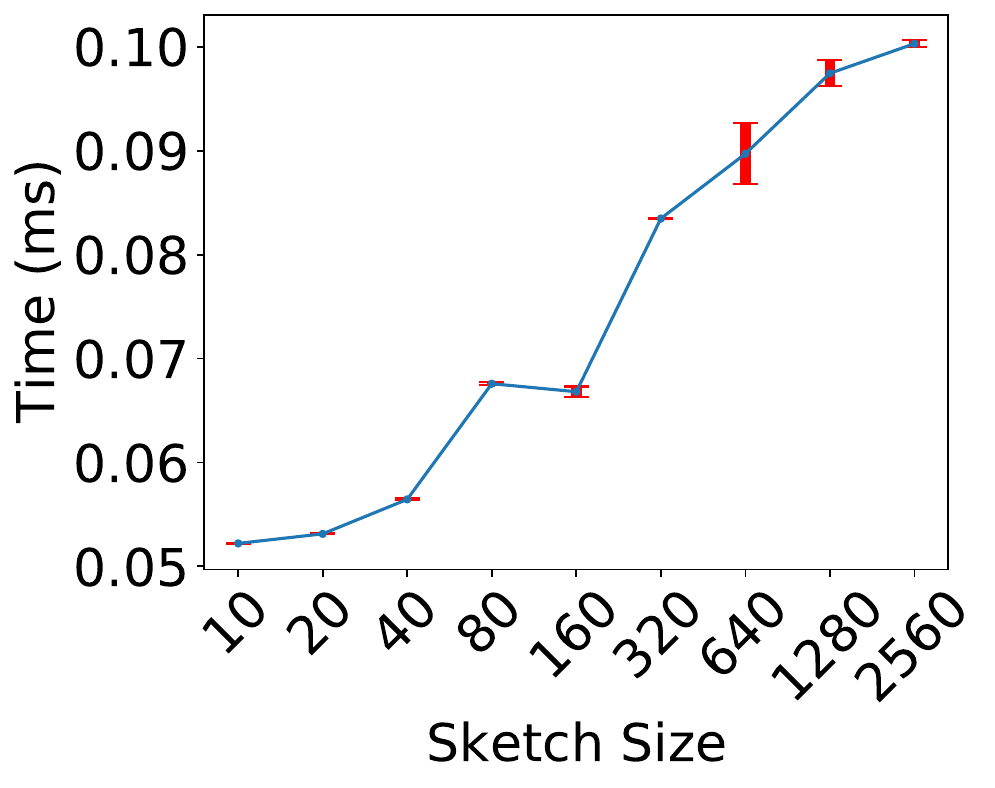}}
\hspace{2mm}
\subfloat[]{\includegraphics[width = 0.25\textwidth]{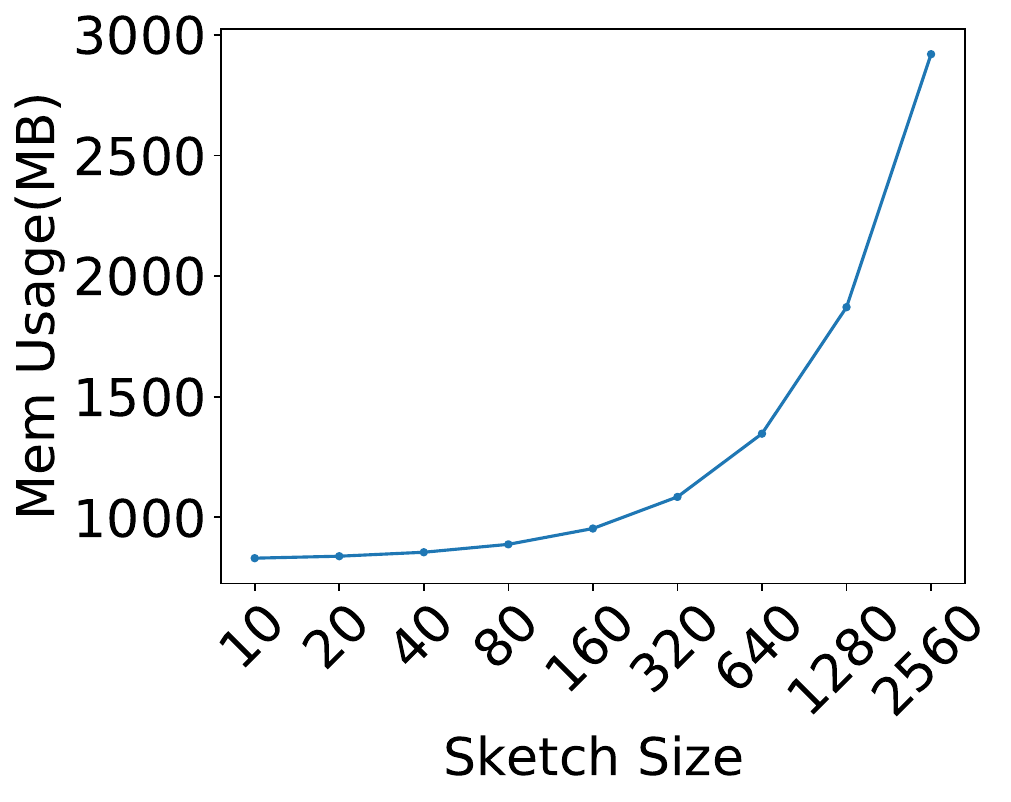}}
\caption{{ Benchmark results for 
(a) \textsc{QueryPair} time under different $m$. (b) Data structure memory usage under different $m$.}}
% \vspace{-4mm}
\label{fig:sketch_size2_app}
\end{figure}
%\vspace{-4mm}

\paragraph{Experiment setup.} We also evaluate our algorithm with $n=10000$ and $d=2560$ uniformly random data points. We run our simulation experiments on a machine with Intel i7-13700K, 64GB memory, Python 3.6.9 and Numpy 1.22.2. We set the number of sketches $L = 10$ and sample $R = 5$ sketches during \textsc{QueryAll}.
We run \textsc{QueryAll} for $10$ times to obtain the average time consumption and the accuracy of \textsc{QueryAll}. We run 
\textsc{QueryPair} for $1000$ times to obtain the average time. The red error bars in the figures represent the standard deviation. We want to answer the following questions:
\begin{itemize}
    \item How is the execution time affected by the sketch size $m$?
    \item How is the query accuracy affected by the sketch size $m$?
    \item How is the memory consumption of our data structure affected by the sketch size $m$?
\end{itemize}

\paragraph{Query Accuracy.} From the part (a) in Figure~\ref{fig:sketch_size_app}, we know that the \textsc{QueryAll} accuracy increases from $80.4\%$ to $90.2\%$ as the sketch size increases from $10$ to $2560$. And we can reach around $90\%$ distance estimation accuracy when the sketch size reaches $160$.

\paragraph{Execution time.} 
From the part (b) in Figure~\ref{fig:sketch_size_app}, \textsc{QueryAll} time grows from $1.05$ second to $214.4$ second as sketch size increases.
From the part (a) in Figure~\ref{fig:sketch_size2_app}, \textsc{QueryPair} time increases from $0.05$ millisecond to $\sim 0.1$ milliseconds as sketch size increases. 
The query time growth follows that larger sketch size leads to higher computation overhead on $\Pi_{j} U q$.
Compared with the baseline whose sketch size is equal to the data point dimension $m=2560$, when sketch size is $160$, the \textsc{QueryAll} is $16.9 \times$ faster and our data structure still achieves $90\%$ estimation accuracy in \textsc{QueryAll}. 

\paragraph{Memory Consumption.} From the part (b) in Figure~\ref{fig:sketch_size2_app}, we find that the memory consumption
increases from $830$MB to $2919$MB as the sketch size increases from $10$ to $2560$. Larger sketch size means that more space is required to store the precomputed $\Pi_j U$ and $\Pi_j U x_i$ matrice. Compared with the baseline sketch size  $m=2560$, when sketch size is $160$, the memory consumption is $3.06 \times$ smaller. 